\def\isarxiv{1}
\documentclass[12pt]{article}
\usepackage[margin=1.25in]{geometry}
\usepackage[T1]{fontenc}
\usepackage{times}

\usepackage{algorithm}
\usepackage{multirow}
\usepackage{colortbl}
\usepackage{nicematrix}
\usepackage{tikz}
\usepackage{makecell}
\usepackage{bbm}
\usepackage{natbib}
\usepackage{times}
\newcommand{\vct}{\boldsymbol }
\newcommand{\sR}{\mathfrak R}

\usepackage[utf8]{inputenc} 
\usepackage[T1]{fontenc}    
\usepackage{hyperref}       
\usepackage{url}            
\usepackage{booktabs}       
\usepackage{amsfonts}       
\usepackage{nicefrac}       
\usepackage{microtype}      
\usepackage{xcolor}         

\usepackage{multirow}

\usepackage{times}

\usepackage{bbold}

\usepackage{booktabs}       
\usepackage{amsfonts}       
\usepackage{nicefrac}       
\usepackage{microtype}      
\usepackage{xcolor}         

\usepackage{amsmath,amsthm,amssymb}
\usepackage{graphicx}
\usepackage{cleveref}
\usepackage{subcaption}
\usepackage{cleveref}
\usepackage{comment}
\usepackage{bbm}
\usepackage{textcomp}
\usepackage{wrapfig}
\usepackage{float}

\usepackage{color}
\usepackage[english]{babel}
\usepackage{graphicx}
\usepackage{grffile}
\usepackage{wrapfig,epsfig}
\usepackage{epstopdf}
\usepackage{url}
\usepackage{color}
\usepackage{epstopdf}
\usepackage{algpseudocode}
\usepackage{comment}
\usepackage{caption}
\usepackage{dsfont}
\usepackage{tikz}

\usepackage{pgfplots}
\pgfplotsset{compat=1.8}
\tikzset{elegant/.style={smooth,thick,samples=500,magenta}}

\theoremstyle{plain}
\newtheorem{theorem}{Theorem}[section]
\newtheorem{lemma}[theorem]{Lemma}
\newtheorem{remark}[theorem]{Remark}

\newtheorem{proposition}[theorem]{Proposition}
\theoremstyle{definition}
\newtheorem{definition}[theorem]{Definition}

\crefname{assumption}{Assumption}{Assumptions}

\ifdefined\usebigfont

\usepackage{times}
\usepackage[fontsize=13pt]{scrextend}
\AtBeginDocument{
	\newgeometry{left=1.56in,right=1.56in,top=1.71in,bottom=1.77in}
}
\else
\fi

\definecolor{b2}{RGB}{51,153,255}
\definecolor{mygreen}{RGB}{80,180,0}

\renewcommand{\hat}{\widehat}
\renewcommand{\tilde}{\widetilde}

\ifdefined\usebigfont

\usepackage{times}
\usepackage[fontsize=13pt]{scrextend}
\AtBeginDocument{
	\newgeometry{left=1.56in,right=1.56in,top=1.71in,bottom=1.77in}
}
\else
\fi

\title{Stochastic Zeroth-Order Optimization under Strongly Convexity and Lipschitz Hessian: Minimax Sample Complexity}

\author{%
Qian Yu\thanks{\texttt{qianyu02@ucsb.edu}. University of California, Santa Barbara.}\and
Yining Wang\thanks{\texttt{yining.wang@utdallas.edu}. University of Texas at Dallas.}\and
Baihe Huang\thanks{\texttt{baihe\_huang@berkeley.edu}. University of California, Berkeley}\and
Qi Lei\thanks{\texttt{ql518@nyu.edu}. New York University.}\and
Jason D. Lee\thanks{\texttt{Jasondl@princeton.edu}. Princeton University.}
}

\begin{document}
\date{}

\maketitle

\maketitle

\begin{abstract}%
 
Optimization of convex functions under stochastic zeroth-order feedback has been a major and challenging question in online learning. In this work, we consider the problem of optimizing second-order smooth and strongly convex functions where the algorithm is only accessible to noisy evaluations of the objective function it queries. 
We provide the first tight characterization for the rate of the minimax simple regret by developing matching upper and lower bounds. 
We propose an algorithm that features a combination of a bootstrapping stage and a mirror-descent stage. 
Our main technical innovation consists of a sharp characterization for 
the spherical-sampling gradient estimator under higher-order smoothness conditions, which allows the algorithm to optimally balance the bias-variance tradeoff, 
and a new iterative method for the bootstrapping stage, which maintains the performance for unbounded Hessian. 
\end{abstract}

\ifdefined\isarxiv

\else

\fi

\section{Introduction}

Stochastic optimization of an unknown function with access to only noisy function evaluations is a fundamental problem in operations research, optimization, simulation and bandit optimization research, commonly known 
as \emph{zeroth-order optimization} \citep{chen2017zoo}, \emph{derivative-free optimization} \citep{conn2009introduction,rios2013derivative} or \emph{bandit optimization} \citep{bubeck2021kernel}. 
In this problem, an optimization algorithm interacts sequentially with an oracle and obtains noisy function evaluations at queried points every time. The algorithm produces an approximately optimal solution after $T$ such evaluations, with its performance evaluated by the expected difference between the function values at the approximate optimal solution produced and the optimal solution. A more rigorous formulation of the problem is given in Sec.~\ref{sec:formulation} below.

Existing works and results on stochastic zeroth-order optimization could be broadly categorized into two classes:
\begin{enumerate}
    \item \textbf{Convex functions}. In the first thread of research, the unknown objective function to be optimized is assumed to be \emph{concave} (for maximization problems) or \emph{convex} (for minimization problems). For these problems, with minimal smoothness (e.g.~objective function being Lipschitz continuous) it is possible to achieve a sample complexity of $\tilde O(\varepsilon^{-2})$ for an expected optimization error or $\varepsilon$, which is also a polynomial function of domain dimension $d$; see for example the works of \cite{agarwal2013stochastic,lattimore2021improved,bubeck2021kernel};
    \item \textbf{Smooth functions}. In the second thread of research, the unknown objective function to be optimized is assumed to be highly \emph{smooth}, but not necessary concave/convex.
    Typical results assume the objective function is H\"{o}lder smooth of order $k\geq 1$, meaning that the $(k-1)$-th derivative of the objective function is Lipschitz continuous.
    Without additional conditions, the optimal sample complexity with such smoothness assumptions is $\tilde O(\varepsilon^{-(2+d/k)})$ \citep{wang2019optimization}, which scales exponentially with the domain dimension $d$.
\end{enumerate}

In this paper, we study the optimal sample complexity of stochastic zeroth-order optimization when the objective function exhibits both (strong) convexity and a high degree of smoothness.
As we have remarked in the first bullet point above, 
with convexity and H\"{o}lder smoothness of order $k=1$ (equivalent to the objective function being Lipschitz continuous), the works of \cite{agarwal2013stochastic,lattimore2021improved,bubeck2021kernel} established an $\tilde O(\varepsilon^{-2})$ upper bound.
With higher order of H\"{o}lder smoothness, i.e., $k=2$ (equivalent to the gradient of the objective being Lipschitz continuous), it is shown that simpler algorithms exist but the sample complexity remains $\tilde O(\varepsilon^{-2})$ \citep{besbes2015non,agarwal2010optimal,hazan2014bandit}, which seemingly suggests the relatively smaller role smoothness plays in the presence of convexity. 
In this paper we show that with even higher order of H\"{o}lder smoothness, i.e., $k=3$ (specifically, the Hessian of the objective being Lipschitz continuous), the optimal sample complexity is improved to $O(\varepsilon^{-1.5})$, which is significantly smaller than the sample complexity of the convex-without-smoothness setting $\tilde O(\varepsilon^{-2})$, or the smooth-without-convexity setting $\tilde O(\varepsilon^{-(2+d/3)})$. More importantly, when the Lipschitzness of Hessian is defined in Frobenius norm (see condition A1), we propose an algorithm that also achieves the optimal dimension dependency, which fully characterizes the optimal sample complexity.

\paragraph{Summary of technical contributions.} {We developed several important techniques in this paper to achieve 
the optimal sample complexity when the objective function is strongly convex and has Lipschitz Hessian. 
First, we show that when estimating the gradient under a stochastic environment, even with an unbounded action space, it could be beneficial to sample with non-isotropic distributions (as opposed to conventional standard Gaussian, or uniform distributions on hyperspheres).
Second, we present a new approach to analyze the bias and variance of the hyperellipsoid-sampling-based gradient estimators, which enables obtaining sharp bounds with tight constants and strengthens the best-known results in the higher-order smoothness case. Third,  we present a two-stage bootstrap-type framework for the algorithmic design, which extends the perturbative analysis in the final stage to the full regime. This extension relies on a non-trivial modification of Newton's method, 
and we proved its robustness under stochastic observation. } 
We complete the characterization of the minimax regret by deriving a lower bound using the KL-divergence-based approach.

\paragraph{Additional related works on higher-order smoothness.}

\begin{table}
\centering
\begin{NiceTabular}{c|c|c}
Lower Bound  & \Block{1-2}{{Upper Bounds}}\\
\hline
\Block[fill=[gray]{0.9}]{4-1}{\makecell{{$\mathbf{\Omega(dT^{-\frac{2}{3}}M^{-1}})$}}} & \citet{bach2016highly}  & \citet{akhavan2020exploiting} \\
 & \makecell{$O(dT^{-\frac{1}{2}}M^{-\frac{1}{2}})$} & \makecell{$O(d^2T^{-\frac{2}{3}}M^{-1})$} \\
\cline{2-3}
 & \citet{novitskii2021improved} &  \Block[fill=[gray]{0.9}]{1-1}{\textbf{Ours}} \\
& \makecell{$O(d^{\frac{5}{3}}T^{-\frac{2}{3}}M^{-1})$} & \Block[fill=[gray]{0.9}]{1-1}{$\mathbf{O(dT^{-\frac{2}{3}}M^{-1}})$} 
\end{NiceTabular}
\caption{The dependence of simple regret on $T$ (number of function evaluations), $d$ (dimension) and $M$ (parameter describing strong convexity). Our results are highlighted in comparison to the prior works.}
    \label{tab:related}
\end{table}
\vspace{-1em}

Recent years have seen increasing attention on exploiting higher order smoothness in bandit optimization. 
Remarkably, it was shown that when the H\"{o}lder smoothness condition holds simultaneously for both $k=2$ and $k=3$, the optimal sample complexity can be improved to $O(\varepsilon^{-1.5})$. 
 \citep{akhavan2020exploiting,novitskii2021improved}. 
We list our results together with the most relevant work in Table \ref{tab:related}. While this line of work also demonstrates the benefit of higher-order smoothness in improving the sample complexity, 
their setting is related but slightly different from what we considered in this work. (See reference therein: \cite{bach2016highly,akhavan2020exploiting,novitskii2021improved}).
 On one hand, the prior work concentrates on projected gradient-descent-like algorithms, which require a Lipschitz gradient (i.e., the $k=2$ requirement, and we do not). This additional requirement can not be removed by simply replacing the gradient steps with Newton's methods, which can lead to unbounded expectation in simple regret in the stochastic case.\footnote{We note that even in the classical analysis of Newton's method, which assumes  zero-error observations, 
the additional $k=2$ smoothness condition was adopted to obtain non-trivial complexity bounds (e.g., see \cite{boyd2004convex}, Section 9.5.3), implying the non-trivialness of removing the $k=2$ smoothness condition. 
In this work, we provided an 
analysis for our proposed bootstrapping algorithm, 
which ensures the achievability of bounded expected regret even with unbounded hessian.}
 On the other hand, their results are based on the generalized H\"{o}lder condition, which is different from our assumption that the Hessian is Lipschitz in Frobenius norm. Therefore we only emphasize the dependence of $d,T$ and $M$ in Table \ref{tab:related} and omit other parameters. We provide a detailed comparison on the implication of these results in Appendix \ref{app:comp}. 


Our results are also related to a special case discussed in \citep{pmlr-v30-Shamir13}, which shows that for \emph{quadratic} functions it is possible to achieve a sample complexity of $\tilde O(\varepsilon^{-1})$. As quadratic functions are infinitely differentiable with bounded derivatives on orders, they are H\"{o}lder smooth of any arbitrary order $k\to\infty$, which could be regarded as an extreme of the results established in this paper which only require $k=3$.

\paragraph{Related works on gradient estimators.} 
Gradient estimation serves as a key building block for stochastic zeroth-order optimization algorithms. 
For instance, a classical one-point estimator was proposed as early as in  
\cite{10.5555/1070432.1070486, blair1985problem}, where the gradient $\nabla f(\vct x)$ is estimated based on empirical measures of $f(\vct x+r\vct u)$ for some fixed $r$ and i.i.d. uniformly random $\vct u$ on the unit hypersphere.  This was later refined to be two-point estimators, 
and the sampling distribution of $\vct u$ was generalized to isotropic distributions such as standard Gaussian (e.g., see \cite{agarwal2010optimal, bach2016highly,  zhang2020boosting}). A majority of prior work focused on the analysis for such estimators 
under the Lipschitz gradient assumption, where the best guaranteed bound for the bias is at the order of $\Theta(r)$, with a polynomial factor dependent on $d$. The line of works by  \cite{bach2016highly,akhavan2020exploiting,novitskii2021improved} also adopted isotropic sampling, and it was shown that with higher-order smoothness of $k=3$, this bound can be improved to $\Theta(r^2)$.
The improvement of sample complexity in our work is mainly due to the tight characterization of our gradient estimator, which covers the special case of isotropic sampling and provides a bound of $\frac{r^2\rho\sqrt{d}}{2(d+2)}$ in the estimation bias. This strengthens or improves the bounds presented in prior works, and a detailed comparison can be found in Appendix \ref{app:comp}.

On the other hand, non-isotropic sampling was used as early as in \cite{0fb2099140344c10bd6305e55866a76f}, then extended in \cite{pmlr-v15-saha11a, hazan2014bandit}. Primarily, they were used to ensure that the sampling points are contained within a bounded action set. 
\citep{pmlr-v134-suggala21a} showed the necessity of non-isotropic sampling over quadratic loss function in the adversarial setting.  
In this work, we essentially demonstrated that non-isotropic sampling can be used to refine a preliminary algorithm by adding a mirror-descent-like final stage.  
More recently, non-isotropic sampling was also adopted in \cite{pmlr-v195-lattimore23a} to optimize convex and global Lipschitz functions. 







\paragraph{Notations.} We follow the 
convention of machine learning theory 
where $\nabla^2 f(\vct x)$ 
denotes the Hessian of $f$ at point $\vct x$, while the trace of Hessian is denoted by $\textup{Tr}\left(\nabla^2 f(\vct x)\right)$. This should not be confused with the notation in classical field theory, where $\nabla^2 f(\vct x)$ instead denotes the trace of the Hessian. 
We use $\|\cdot\|_2$ to denote vector $\ell_2$ norms, and $\|\cdot\|_{\textup{F}}$ to denote matrix Frobenius norms.
We use $I_d$ to denote the identity matrix, and $S^{d-1}$ to denote the unit hypersphere centered at the origin, both for the $d$-dimensional Euclidean space $\mathbb{R}^{d}$. We adopt the conventional notations (i.e., $O$, $\Omega$, $o$, and $\omega$) to describe regret bounds in the asymptotic sense with respect to the total number of samples (denoted by $T$). 

\section{Problem Formulation}\label{sec:formulation}

We consider the stochastic optimization problem under the class of functions that are strongly convex and have Lipschitz Hessian. The goal in this setting is to design learning algorithms to achieve approximately the global minimum of an unknown objective function $f: \mathbb{R}^d\rightarrow \mathbb{R}$. 

A learning algorithm $\mathcal{A}$ can interact with the function by adaptively sampling their value for $T$ times, and receive noisy observations. At each time $t\in [T]$, the algorithm selects $\boldsymbol{x}_t\in \mathbb{R}^d$, and receives the following observation, 
\begin{align}
    y_t=f(\boldsymbol{x}_t) +w_t,
\end{align}
where $\{w_t\}_{t=1}^{T}$ are independent random variables with zero mean and bounded variance. Formally, the algorithm can be described by a list of conditional distributions where each $\vct x_t$ is selected based on all historical data $\{\vct x_\tau, y_{\tau}\}_{\tau<t}$ and the corresponding distribution. Then for any $t$, we assume that $\mathbb{E}[w_t|\{\vct x_\tau, y_{\tau}\}_{\tau<t}, \vct x_t]=0$ and $\textup{Var}[w_t|\{\vct x_\tau, y_{\tau}\}_{\tau<t}, \vct x_t]\leq 1$ for any $t$.\footnote{If the variances of $w_t$'s are bounded by a different constant, all our results can be reproduced by normalizing the values of $f$.} 
For simplicity, we also adopt a common assumption that the additive noises are subgaussian, particularly, $\mathbb{P}[|w_t|>s|\{\vct x_\tau, y_{\tau}\}_{\tau<t}, \vct x_t]\leq 2e^{-{s^2}}$ for all  $s>0$ and $t\in[T]$. However, the subgaussian assumption can be removed by adopting more sophisticated mean-estimation methods (e.g., see \cite{nemirovskii1983problem, JERRUM1986169, ALON1999137, 9719860,NEURIPS2023_6e60a902}).



We assume that the objective function $f$ is second-order differentiable. Furthermore, we impose the following conditions.  
\begin{enumerate}
    \item[(A1)] (Lipschitz Hessian). There exist a constant  $\rho\in(0,+\infty)$ such that for all $\vct x, \vct x'\in\mathbb R^d$, it holds that $\|\nabla^2 f(\boldsymbol{x})-\nabla^2 f(\boldsymbol{x}')\|_{\textup{F}}
\leq \rho \|\boldsymbol{x}'-\boldsymbol{x}\|_2$, where $\|\cdot\|_{\textup{F}}$ denotes the Frobenius norm;
    \item[(A2)] (Strong Convexity). There exists a constant $M\in(0,+\infty)$ such that for any $\vct x\in\mathbb R^d$, 
    the minimum eigenvalue of the Hessian $\nabla^2 f(\boldsymbol{x})$ is greater than $M$.
    \item[(A3)] \label{assump:a3} (Bounded Distance from Initialization to Optimum Point). There exists a constant $R\in(0,+\infty)$ such that the infimum of $f(\vct x)$ within the hyperball $\|\boldsymbol{x}\|_2\leq R$  
    is identical to the infimum of $f(\vct x)$ over the entire $\mathbb R^d$.
\end{enumerate}



In the rest of this paper, we let $\mathcal{F}(\rho,M,R)$ denote the set of all second-order differentiable functions that satisfy the above conditions, with corresponding constants given by $\rho,M$, and $R$. 
We aim to find algorithms to achieve asymptotically the following minimax simple regret, which measures the expected difference of the objective function on $x_T$ and the optimum.
$$
\sR(T;\rho,M,R) := \inf_{\mathcal A}\sup_{f\in\mathcal F(\rho,M,R)}\mathbb E\left[f(\vct x_T)-f(\vct x^*)\right],
$$
where $\vct x^*$ denotes the global minimum point of $f$.

\section{Main Results}

\begin{theorem}\label{thm1}
For any dimension $d$ and constants $\rho,M,R$, the minimax simple regrets are  upper bounded by $\limsup_{T\rightarrow \infty}\sR(T;\rho,M,R)\cdot T^{\frac{2}{3}}\leq C \cdot \left(\frac{\rho^{\frac{2}{3}}}{M}d \right)$, where $C$ is a universal constant. 
\end{theorem}

\begin{theorem}\label{thm2}
For any fixed dimension $d$ and constants $\rho,M,R$, the minimax simple regrets are lower bounded by $\liminf_{T\rightarrow \infty}\sR(T;\rho,M,R)\cdot T^{\frac{2}{3}}\geq C \cdot \left(\frac{\rho^{\frac{2}{3}}}{M}d \right)$ 
when the additive noises $w_1,...,w_T$ are standard Gaussian, where $C$ is a universal constant. 
\end{theorem}

\section{Proof Ideas for Theorem \ref{thm1}}

The proposed algorithm 
operates in two stages (see Algorithm \ref{alg:1gd3_opt}).
In the first stage, the algorithm uses a small fraction of samples to obtain a rough estimation of the global minimum point. We ensure that the estimation in the first stage is sufficiently accurate with high probability, so that in the following final stage, the objective function can be approximated by a quadratic function and the resulting approximation error can be bounded using tensor analysis. 

\subsection{Key Techniques and The Final Stage}

We first present the key steps of our algorithm, which relies on the subroutines presented in Algorithm \ref{alg:1gd}-\ref{alg:1gd3}, i.e., GradientEst, BootstrappingEst, and HessianEst. These subroutines estimate the (linearly transformed) gradients and Hessian functions of $f$ at any given point by sampling the values of $f$ on hyperellipsoids. The key ingredient of our proof is the sharp characterizations for the biases and variances of the GradientEst estimator, stated in Theorem  \ref{thm:gdest}.

\begin{algorithm}
\caption{GradientEst}\label{alg:1gd}
\begin{algorithmic}

\State {\bf Input}: {$\boldsymbol{x},Z,n$} \Comment{$Z$ is a ${d\times d}$ matrix, return $\hat{\vct{g}}$ as an estimator of $Z\nabla f(\vct x)$}
\For{$k\leftarrow  1 $ to $n$}

\State Let $\vct u_k$ be a point sampled uniformly randomly from the standard hypersphere $S^{d-1}$ %
\State Let 
let $y_+$, $y_-$ be samples of $f$ at $\boldsymbol{x}+Z\boldsymbol{u}_k$ and $\boldsymbol{x}-Z\boldsymbol{u}_k$,  respectively, 
let  $\vct g_k=\frac{d}{2}(y_+-y_-)\boldsymbol{u}_k$
\EndFor

\State {\bf Return} $\hat{\vct{g}}=\frac{1}{n}\sum_{k=1}^{n} \vct g_k$

\end{algorithmic}
\end{algorithm}
\begin{algorithm}
\caption{BootstrappingEst}\label{alg:1gd2}
\begin{algorithmic}

\State {\bf Input}: {$\boldsymbol{x},r,n$}\Comment{Goal: estimate  $\nabla f(\vct x)$ coordinate wise with $O(nd)$ samples}

\State Let $\vct e_1,...,\vct e_d$ be any orthonormal basis of $\mathbb{R}^d$ 
\For{$k\leftarrow  1 $ to $d$} 

\State Let $y_{+,k}$, $y_{-,k}$ each  
be the average of $n$ samples of $f$ at $\boldsymbol{x}+r\boldsymbol{e}_k$ and $\boldsymbol{x}-r\boldsymbol{e}_k$ respectively

\State Let $m_k=
{(y_+-y_-)}/{2r}$ \Comment{Estimate the $k$th entry} 
\EndFor
\State {\bf Return}   $\hat{\vct{m}}=\{m_k\}_{k\in[d]}$
\end{algorithmic}
\end{algorithm}
\begin{algorithm}
\caption{HessianEst}\label{alg:1gd3}
\begin{algorithmic}

\State {\bf Input}: {$\boldsymbol{x},r,n$}\Comment{Goal: estimate $\nabla^2 f(\vct x)$ coordinate wise with $O(nd^2)$ samples}

\State Let $\vct e_1,...,\vct e_d$ be any orthonormal basis of $\mathbb{R}^d$ 
\State Let $y$ be the average of $n$ samples of $f$ at $\boldsymbol{x}$
\For{$k\leftarrow  1 $ to $d$} 

\State Let $y_{+,k}$, $y_{-,k}$ each  
be the average of $n$ samples of $f$ at $\boldsymbol{x}+r\boldsymbol{e}_k$ and $\boldsymbol{x}-r\boldsymbol{e}_k$ respectively

\State Let $H_{kk}=
{(y_++y_--2y)}/{r^2}$ \Comment{Diagonal entries} 
\For{$\ell \leftarrow  k+1 $ to $d$} 

\State Let $H_{k \ell }=H_{\ell k}$ be the average of $n$ samples of $(f(\boldsymbol{x}+r\boldsymbol{e}_k+r\boldsymbol{e}_\ell)+f(\boldsymbol{x}-r\boldsymbol{e}_k-r\boldsymbol{e}_\ell)-$\\
\hspace{6mm} $f(\boldsymbol{x}+r\boldsymbol{e}_k-r\boldsymbol{e}_\ell)-f(\boldsymbol{x}-r\boldsymbol{e}_k+r\boldsymbol{e}_\ell))/4r^2$ \Comment{ Off-diagonal entries} 
\EndFor
\EndFor
\State Let $\hat{H}_0=\{H_{jk}\}_{(i,j)\in [d]^2}$, and $\hat{H}$ be the matrix with same eigenvectors but with each eigenvalue $\lambda$ replaced by $\max\{\lambda,M\}$ \Comment{Projecting to the set where $\hat{H}-M I_d$ is positive semidefinite}

\State {\bf Return}    $\hat{H}$
\end{algorithmic}
\end{algorithm}

\begin{theorem}\label{thm:gdest}
For any fixed inputs $\boldsymbol{x}$, $Z$, $n$, and any function $f$ satisfying the Lipschitz Hessian condition with parameter $\rho$, the output $\hat{\boldsymbol{g}}$ returned by the GradientEst subroutine 
satisfies the following properties
\begin{align}
    &~ ||\mathbb{E}[\hat{\boldsymbol{g}}]-Z\nabla f (\boldsymbol{x}) ||_2 \leq \frac{\lambda_Z^3\rho \sqrt{d}}{2(d+2)}, \label{tteq:1}\\
    &~ \textup{Tr}\left(
      \textup{Cov}[\hat{\boldsymbol{g}}]\right) \leq\frac{2d}{n}||Z\nabla f(\boldsymbol{x})||_2^2 +\frac{d^2}{18n}\left(\rho \lambda_Z^3 \right)^2+\frac{d^2}{2n}\label{tteq:2},
\end{align}
where $\lambda_Z$ is the largest singular value of $Z$. 
\end{theorem}

\begin{remark}
Inequality \eqref{tteq:1} provides a sharp characterization for the bias of the gradient estimator, as it can be matched for any $\lambda_Z$ and $d$ with a cubic polynomial $f$. 
Inequality \eqref{tteq:2} is sharp in the asymptotic regime when both $\nabla f$ and $\lambda_Z$ approaches zero.
\end{remark}

We also provide rough estimates on the high-probability bounds for the BootstrappingEst and the HessianEst functions. Specifically, we show that their errors have sub-Gaussain tails in distribution, as stated in the following theorem. 


\begin{theorem}\label{thm:hest}\label{thm:pgest}
For any fixed inputs $\boldsymbol{x}$, $r$, $n$, any function $f$ satisfying the Lipschitz Hessian condition with parameter $\rho$, and any variable $K>0$, 
the outputs $\hat{\vct m}$ and $\hat{H}$ returned by the BootstrappingEst and the HessianEst subroutine 
satisfy the following conditions. 
\begin{align}
\mathbb{P}\left[\left|\left|\hat{\boldsymbol{m}}-\nabla f(\vct x)\right|\right|_2\geq K\right] \leq &  \  2\exp\left({-\frac{K^2}{\frac{3d\rho^2r^4}{4}+\frac{12d}{nr^2}}}\right),\label{eq431}\\
\mathbb{P}\left[\left|\left|\hat{H}-\nabla^2 f(\vct x)\right|\right|_{\textup{F}}\geq  K \right]& \leq  2\exp\left({-\frac{K^2}{2{d^2\rho^2 r^2}+\frac{144d^2}{nr^4}}}\right).\label{eq432}
\end{align} 
\end{theorem}
We postpone the proof of the above theorems to Section \ref{sec:ptgdest} and Appendix \ref{sec:pthest} and proceed to describe how these results are used in the algorithm. 


For brevity, 
let $\epsilon\triangleq \frac{\rho^{\frac{2}{3}}}{M}d T^{-\frac{2}{3}}$ be the minimax regret we aim to achieve, and let $\vct x_{\textup{B}}$ denote the estimator $\vct x$ stored at the end of the first stage. The role of the final stage is to ensure that if $f(\vct x_{\textup{B}})-f(\vct x^*)$ is sufficiently small with high probability, 
the final result of the proposed algorithm achieves the stated simple regret guarantees.
Formally, we require that 
\begin{align}
\lim_{T\rightarrow \infty}\sup_{f\in\mathcal{F}(\rho, M, R)}\mathbb{E}\left[\left(f(\vct x_{\textup{B}})-f(\vct x^*)\right)^{\frac{3}{2}}\right]/\epsilon =0.    
   \label{eq:fststbd}
\end{align}
Note that the above condition implies that $f(\vct x_{\textup{B}})-f(\vct x^*)$ concentrates below $o\left(\epsilon^{\frac{2}{3}}\right)$, which is weaker than the $O(\epsilon)$ rates stated in our main theorems.\footnote{With more sophisticated analysis, this concentration requirement can be improved to only requiring a similar upper bound of $o\left(\epsilon^{\frac{1}{2}}\right)$. However, we choose equation \eqref{eq:fststbd} to provide a simpler proof, as it does not affect the asymptotic sample complexity. } The bottleneck of the overall algorithm is 
on the final stage, and one can achieve equation \eqref{eq:fststbd} using any suboptimal algorithm with an expected simple regret of $o(T^{-\frac{4}{9}})$. For example, one can run the suboptimal algorithm twice, estimate their achieved function values by averaging over $o(T)$ samples, 
and then choose the outcome with the smaller estimated function value as $\vct x_{\textup{B}}$. 
In the rest of this section, we prove Theorem \ref{thm1} assuming the correctness of equation \eqref{eq:fststbd}. A self-contained proof for equation \eqref{eq:fststbd} is provided in Appendix \ref{app:gup}. 



\begin{algorithm}
\caption{An Example Algorithm to Achieve the Minimax Rates}\label{alg:1gd3_opt}

\begin{algorithmic}
\State{\bf Input} {$T,\rho, M$}
\State Let $\vct{x}=\boldsymbol{0}$
\State The First Stage:


\For{$k\leftarrow  1 $ to $\lfloor T^{0.1} \rfloor$}

\State Let  $n_{\textup{m}}=\lfloor \frac{{T}^{0.9}}{10d} \rfloor$, $n_{\textup{H}}=\lfloor \frac{{T}^{0.9}}{10d^2} \rfloor$,  
 $r_{\textup{m}}= \left(\frac{8}{n_{\textup{m}}\rho^2}\right)^{\frac{1}{6}}$, $r_{\textup{H}}= \left(\frac{144}{n_{\textup{H}}\rho^2}\right)^{\frac{1}{6}}$
\State Let $\hat{\vct m}=$ BootstrappingEst$(\boldsymbol{x},r_{\textup{m}},n_{\textup{m}})$, $\hat{H}=$ HessianEst$(\boldsymbol{x},r_{\textup{H}},n_{\textup{H}})$ 
\State 
Let $H_{{m}^*}$ denote the matrix with the same eigenvectors of $\hat{H}$ but each eigenvalue $\lambda$ replaced by $\max\{\lambda, {m}^*\}$, choose ${m}^*$ to be the smallest value such that $||H_{{m}^*}^{-1}\hat{\vct m}||_2\leq \frac{M}{\rho}$. 
\State 
Let 
 $\vct x=\vct x-H_{{m}^*}^{-1}\hat{\vct m}$
\EndFor

\State The Final Stage:
\State Let 
$n_{\textup{g}}=\lfloor \frac{{T}}{10} \rfloor$, $n_{\textup{H}}=\lfloor \frac{{T}}{10d^2} \rfloor$,  
 $r_{\textup{g}}= \left(\frac{d^3}{n_{\textup{g}}\rho^2}\right)^{\frac{1}{6}}$, $r_{\textup{H}}= \left(\frac{144}{n_{\textup{H}}\rho^2}\right)^{\frac{1}{6}}$
 \State Let $\hat{H}=$ HessianEst$(\boldsymbol{x},r_{\textup{H}},n_{\textup{H}})$, $Z_{H}$ be any symmetric matrix such that $Z_{H}^2=\hat{H}^{-1}$, and $\lambda_{Z_{H}}$ be the largest eigenvalue of $Z_{H}$
\State Let $Z=r_{\textup{g}}Z_H/\lambda_{Z_{H}}$, $\hat{\vct g}=$ GradientEst$(\vct{x},Z,n_{\textup{g}})$, $\vct r=-\hat{H}^{-1}Z^{-1}\hat{\vct g}$



\State Project $\vct r$ to the $L_2$ ball of radius $\frac{M}{\rho}$,  
  i.e., $\vct r= \vct r \cdot \min\{1, \frac{M}{\rho||\vct r||_2}\}$
\State {\bf Return} $\vct x=\vct x+\vct r$ 

\end{algorithmic}
\end{algorithm}



Before proceeding with the proof, we provide a high-level description of the algorithm in the final stage. 
At the beginning, we perform a Hessian estimation near $\vct x_{\textup{B}}$ using the HessianEst subroutine with $O(T)$ samples. From Theorem \ref{thm:hest},
our choice of parameters results in an expected estimation error of 
$o(1)$ for sufficiently large $T$. 


The algorithm proceeds 
to find a real matrix $Z_H$, 
which essentially serves as a linear transformation on the action domain such that the Hessian of the transformed function is approximately the identity matrix. Note that the projection step in the HessianEst function ensures the eigenvalues of the estimator are no less than $M$.  There is always a valid solution of $Z_H$. 

Then, we estimate the gradient at $\vct x_{\textup{B}}$ using the GradientEst subroutine, which samples on a hyperellipsoid with a shape characterized by  $Z_H$. We chose the hyperellipsoid sampling in the final stage due to its superior performance in the small-gradient regime compared to coordinate-wise sampling. In contrast, the coordinate-wise estimator is used in the bootstrapping stage to eliminate the dependency of 
the local gradient on its bias-variance tradeoff, which is beneficial for the non-asymptotic analysis.
Particularly, we scale the hyperellipsoid with a carefully designed factor (see the definition of variable $r_{\textup{g}}$) to minimize the estimation error. Then, the remaining steps can be interpreted as a modified Newton step, which essentially approximates the global minimum point with a quadratic approximation.  

The analysis in our proof relies on the following proposition, which is proved in Appendix \ref{app:pp_ptech1}. 
\begin{proposition}\label{pp:ptech1}
For any given point $\vct x_\textup{B}$ and any function $f$ that satisfies strong convexity and Lipschitz Hessian, let $\tilde{\vct x}\triangleq \vct x_\textup{B}-(\nabla^2 f(\vct x_\textup{B}))^{-1}\nabla f(\vct x_\textup{B})$ and $\tilde{f}(\vct x)$ denote the quadratic approximation $\frac{1}{2} (\vct x- \tilde{\vct x})^{\intercal}\nabla^2 f(\vct x_\textup{B})(\vct x- \tilde{\vct x})$, we have the following inequality for all  $\vct x$ with $||\vct x-\vct x_{\textup{B}}||_2\leq \frac{M}{\rho}$.
\begin{align}
f(\vct x)- f^*\leq&\ 2\tilde{f}(\vct x) + \frac{12\rho (f(\vct x_{\textup{B}})- f^*)^{\frac{3}{2}}}{M^{\frac{3}{2}}}.\label{ineq:tech1eq6}  
\end{align}
Furthermore, if $\vct x$ is generated by the final stage of Algorithm \ref{alg:1gd3_opt} with any parameter values that satisfy $n_{\textup{g}}\geq d^3$, $n_{\textup{H}} \geq  \frac{64\rho^4 d^6}{M^6}$ and the first-stage output is set to $\vct x_{\textup{B}}$,    
then 
\begin{align}\label{ineq:tech1eq7}
\mathbb{E}\left[\tilde{f}(\vct x)\ |\ \vct x_{\textup{B}}\right]\leq 
  \left(\frac{14d^2\rho^{\frac{4}{3}}}{M^2n_{\textup{H}}^{\frac{1}{3}}}+\frac{52d}{n_{\textup{g}}}\right)(f(\vct x_{\textup{B}})-f(\vct x^*)) +\frac{82\rho}{M^{\frac{3}{2}}} (f(\vct x_{\textup{B}})-f(\vct x^*))^{\frac{3}{2}} +\frac{3d\rho^{\frac{2}{3}}}{Mn_{\textup{g}}^{\frac{2}{3}}}. 
\end{align}
\end{proposition}

Now, we use Proposition \ref{pp:ptech1} to prove the achievability result. 
\begin{proof}[Proof of Theorem \ref{thm1} given inequality \eqref{eq:fststbd}.]
First, recall our construction ensures that 
$||\vct x_T-\vct x_{\textup{B}}||_2\leq \frac{M}{\rho}$. Inequality \eqref{ineq:tech1eq6} can always be applied and we have
\begin{align*}
\sR(T;\rho,M,R) \leq&\ \sup_{f\in\mathcal{F}(\rho, M, R)}\mathbb E\left[2\tilde{f}(\vct x) + \frac{12\rho (f(\vct x_{\textup{B}})- f^*)^{\frac{3}{2}}}{M^{\frac{3}{2}}}\right].  
\end{align*}
Then, when $T$ is sufficiently large, the conditions of  \eqref{ineq:tech1eq7} holds and we have
\begin{align*}
\sR(T;\rho,M,R) \leq&\ \sup_{f\in\mathcal{F}(\rho, M, R)}\mathbb E\left[ \left(\frac{28d^2\rho^{\frac{4}{3}}}{M^2n_{\textup{H}}^{\frac{1}{3}}}+\frac{104d}{n_{\textup{g}}}\right)(f(\vct x_{\textup{B}})-f(\vct x^*)) + \frac{176\rho (f(\vct x_{\textup{B}})- f^*)^{\frac{3}{2}}}{M^{\frac{3}{2}}}\right]\\&+\frac{6d\rho^{\frac{2}{3}}}{Mn_{\textup{g}}^{\frac{2}{3}}}.
\end{align*}
Note that inequality \eqref{eq:fststbd} implies that $\mathbb{E}\left[f(\vct x_{\textup{B}})-f(\vct x^*)\right]=o(\epsilon^{\frac{2}{3}})=o(T^{-\frac{4}{9}})$ and we have that $n_{\textup{H}}^{-\frac{1}{3}}+n_{\textup{g}}^{-1}=o(T^{-\frac{2}{9}})$. The RHS of the above inequality is dominated by the last term. Hence,
\begin{align*}
\limsup_{T\rightarrow \infty}\sR(T;\rho,M,R)\cdot T^{\frac{2}{3}}\leq \limsup_{T\rightarrow \infty}\frac{6d\rho^{\frac{2}{3}}T^{\frac{2}{3}}}{Mn_{\textup{g}}^{\frac{2}{3}}} =O\left(\frac{\rho^{\frac{2}{3}}}{M}d \right).
\end{align*}
\end{proof}


\subsection{Proof of Theorem \ref{thm:gdest}}\label{sec:ptgdest}

To prove inequality \eqref{tteq:1}, we investigate the following function
$$\boldsymbol{G}(r;\boldsymbol{x})\triangleq \mathbb{E}_{\boldsymbol{u}\sim\textup{Unif}(S^{d-1})}\left[\frac{d}{2r}(f(\boldsymbol{x}+r\boldsymbol{u})-f(\boldsymbol{x}-r\boldsymbol{u}))\boldsymbol{u}\right],$$
where 
$\textup{Unif}(S^{d-1})$ denotes the uniform distribution on  $S^{d-1}$. Recall that in our algorithm we have $\mathbb{E}[\hat{\vct g}]=r\boldsymbol{G}(r;\boldsymbol{x})$ if $Z=rI_d$ for some $r\in\mathbb (0,+\infty)$, and by differentiability we have $\nabla f (\boldsymbol{x})=\lim_{z\rightarrow 0^+} \boldsymbol{G}(z;\boldsymbol{x})$. Under this condition, we can bound $||\mathbb{E}[\hat{\vct g}]-r\nabla f (\boldsymbol{x}) ||_2$  by integration, i.e.,
 \begin{align}
     ||\mathbb{E}[\hat{\vct g}]-r\nabla f (\boldsymbol{x}) ||_2=&~ r\left|\left|\boldsymbol{G}(r;\boldsymbol{x})- \lim_{z\rightarrow 0^+} \boldsymbol{G}(z;\boldsymbol{x})\right|\right|_2 
     \leq ~ r\int_{0^+}^{r} \left|\left|\frac{d}{dz}\boldsymbol{G}(z;\boldsymbol{x})\right|\right|_2 dz. \label{ineq:upper_lag} 
\end{align}

Note that $\boldsymbol{G}(z;\boldsymbol{x})$ can be written into the following equivalent form. 
$$\boldsymbol{G}(z;\boldsymbol{x})=\frac{\int_{S^{d-1}} \frac{d}{2z}(f(\boldsymbol{x}+z\boldsymbol{u})-f(\boldsymbol{x}-z\boldsymbol{u})) {\bf dA}}{\int_{S^{d-1}} ||{\bf dA}||_2} ,$$
where the integration is with respect to $\boldsymbol{u}$ over the surface $S^{d-1}$, and ${\bf dA}$ is the vector surface element, i.e.,  
with the magnitude being the infinitesimally small surface area and the direction
perpendicular to the surface (pointing outward). The differential of $\boldsymbol{G}(z;\boldsymbol{x})$ over $z$ can be written as
\begin{align*}
    \frac{d}{dz}\boldsymbol{G}(z;\boldsymbol{x})
    =&~\frac{\int_{S^{d-1}} \frac{\partial}{\partial{z}}\left(\frac{d}{2z}(f(\boldsymbol{x}+z\boldsymbol{u})-f(\boldsymbol{x}-z\boldsymbol{u}))\right) {\bf dA}}{\int_{S^{d-1}} ||{\bf dA}||_2}\\
    = &~
    \frac{\int_{S^{d-1}} -\frac{d}{2z^2}\left(f(\boldsymbol{x}+z\boldsymbol{u})-f(\boldsymbol{x}-z\boldsymbol{u})\right){\bf dA}}{\int_{S^{d-1}} ||{\bf dA}||_2}\\&+ \frac{\int_{S^{d-1}} \frac{d}{2z}\boldsymbol{u}\cdot (\nabla f(\boldsymbol{x}+z\boldsymbol{u})+\nabla f(\boldsymbol{x}-z\boldsymbol{u})) {\bf dA}}{\int_{S^{d-1}} ||{\bf dA}||_2}.
    \end{align*}
   The gist of this proof is to note that for any $\vct u\in S$ we have $ \vct u$ and $\bf dA$ are parallel (i.e., $ \vct u$ is parallel to the normal vector of the hypersphere at the same point),  so the second term in the integral above on the numerator can be written as
   \begin{align*}
       \int_{S^{d-1}}\frac{d}{2z}\boldsymbol{u} (\nabla f(\boldsymbol{x}+z\boldsymbol{u})+\nabla f(\boldsymbol{x}-z\boldsymbol{u}))\cdot {\bf dA}.
   \end{align*}
   Hence, by divergence theorem, we have 
    \begin{align}\label{eq:upper_4}
         \frac{d}{dz}\boldsymbol{G}(z;\boldsymbol{x})  =&~\frac{1}{\int_{S^{d-1}} ||{\bf dA}||_2} \cdot 
         \bigg(\int_{B^d} \nabla_{\vct{u}}\cdot\Big(-\frac{d}{2z^2}I_d\left(f(\boldsymbol{x}+z\boldsymbol{u})-f(\boldsymbol{x}-z\boldsymbol{u})\right)\notag\\
        &~ +\left(\nabla f(\boldsymbol{x}+z\boldsymbol{u})+\nabla f(\boldsymbol{x}-z\boldsymbol{u})\right)\frac{d}{2z}\boldsymbol{u}\Big)  {\bf dV}\bigg)\nonumber\\
    = &~ \frac{d}{2}\cdot \frac{\int_{B^d}\boldsymbol{u} \, \textup{Tr}( \nabla^2 f(\boldsymbol{x}+z\boldsymbol{u}) -\nabla^2 f(\boldsymbol{x}-z\boldsymbol{u})){\bf dV}}{\int_{S^{d-1}} ||{\bf dA}||_2},
\end{align}
where 
$B^d$ denotes the standard hyperball. 

Now consider any unit vector $\vct e$. Let ${\boldsymbol{u_e}}$ denote the reflection of $\vct u$ with respect to the hyperplane orthogonal to $\vct e$, i.e., $\boldsymbol{u_e}\triangleq \vct u-2 (\vct u\cdot \vct e)\vct e$. Because the hyperball $B$ is invariant under the reflection $\boldsymbol{u}\rightarrow \boldsymbol{u_e}$, equation \eqref{eq:upper_4} can also be written as 
    \begin{align}\label{eq:upper_5x}
        \frac{d}{dz}\boldsymbol{G}(z;\boldsymbol{x}) 
        = &~ \frac{d}{2}\cdot\frac{\int_{B^d} \boldsymbol{u_e}\, \textup{Tr}( \nabla^2 f(\boldsymbol{x}+z\boldsymbol{u_e}) -\nabla^2 f(\boldsymbol{x}-z\boldsymbol{u_e})){\bf dV}}{\int_{S^{d-1}} ||{\bf dA}||_2}.
\end{align}
Hence, by averaging equation \eqref{eq:upper_4} and \eqref{eq:upper_5x}, we have
   \begin{align}
        \frac{d}{dz}\boldsymbol{G}(z;\boldsymbol{x}) \cdot \vct e 
        = &~\frac{d}{4}\frac{\int_{B^d}\boldsymbol{u}\, \textup{Tr}( \nabla^2 f(\boldsymbol{x}+z\boldsymbol{u}) -\nabla^2 f(\boldsymbol{x}-z\boldsymbol{u})){\bf dV}}{\int_{S^{d-1}} ||{\bf dA}||_2}\cdot \vct e\nonumber\\
        &~ +\frac{d}{4}\frac{\int_{B^d}\boldsymbol{u_e}\, \textup{Tr}( \nabla^2 f(\boldsymbol{x}+z\boldsymbol{u_e}) -\nabla^2 f(\boldsymbol{x}-z\boldsymbol{u_e})){\bf dV}}{\int_{S^{d-1}} ||{\bf dA}||_2}\cdot \vct e\nonumber\\
        =&~\frac{d}{4}\frac{ \int_{B^d}\boldsymbol{u}\cdot \vct e\, \textup{Tr}( \nabla^2 f(\boldsymbol{x}+z\boldsymbol{u}) -\nabla^2 f(\boldsymbol{x}+z\boldsymbol{u_e})){\bf dV}}{\int_{S^{d-1}} ||{\bf dA}||_2}\notag\\
        &~ + \frac{d}{4}\frac{\int_{B^d}-\boldsymbol{u}\cdot \vct e \, \textup{Tr}(\nabla^2 f(\boldsymbol{x}-z\boldsymbol{u})  -\nabla^2 f(\boldsymbol{x}-z\boldsymbol{u_e})){\bf dV}}{\int_{S^{d-1}} ||{\bf dA}||_2}.\label{eq:upper_5}
\end{align}
By the Lipschitz Hessian condition and Cauchy's inequality, the difference between the differential terms above can be bounded as follows.
\begin{align}
 \left| \textup{Tr} \left(\nabla^2 f(\boldsymbol{x}\pm z\boldsymbol{u}) -\nabla^2 f(\boldsymbol{x}\pm z\boldsymbol{u_e})\right)\right| 
\leq&~\sqrt{d} ||\nabla^2 f(\boldsymbol{x}\pm z\boldsymbol{u})- \nabla^2 f(\boldsymbol{x}\pm z\boldsymbol{u_e})||_{\textup{F}}\nonumber\\
\leq &~ \rho \sqrt{d} ||z\boldsymbol{u}-z\boldsymbol{u_e}||_2 
= 2z\rho \sqrt{d}  |\vct u\cdot \vct e|.\label{ineq:upper_6}
\end{align}
Consequently,  
 \begin{align*}
      \left|\frac{d}{dz}\boldsymbol{G}(z;\boldsymbol{x})\cdot \vct e\right|
        &\leq \frac{ z\rho d\sqrt{d}\int_{B^d} \left(\boldsymbol{u}\cdot \vct e\right)^2 {\bf dV}}{\int_{S^{d-1}} ||{\bf dA}||_2}
        =\frac{z\rho \sqrt{d}}{d+2}.
\end{align*}
Note that $\vct e$ can be any unit vector. We have essentially bounded the $\ell_2$ norm of  $\frac{d}{dz}\boldsymbol{G}(z;\boldsymbol{x})$, i.e.,
\begin{align*}
    \left|\left|\frac{d}{dz}\boldsymbol{G}(z;\boldsymbol{x})\right|\right|_2\leq \frac{z\rho \sqrt{d}}{d+2}.
\end{align*} 
As mentioned earlier, when $Z=rI_d$ inequality \eqref{tteq:1} is obtained by applying this gradient-norm bound to inequality \eqref{ineq:upper_lag} .

For general input matrix $Z$, we can view GradientEst as a subroutine that operates on the same function $f$ but with a linear transformation applied to the input domain. Formally, let $f'(\vct y)\triangleq f(\vct x + \frac{Z}{\lambda_Z} (\vct y-\vct x))$. We have that $f'$ satisfies the Lipschitz Hessian condition with parameter $\rho$ as well. Therefore, inequality \eqref{tteq:1} can be obtained following the same analysis by replacing $f$ with $f'$ and $Z$ with $\lambda_Z I_d$. 

Now we present the proof for inequality \eqref{tteq:2}. Formally, let $w_+$, $w_-$ be two independent samples of additive noises. Then the trace of covariance matrix of $\hat{\vct g}$ can upper bounded using the second moments of single measurements.
\begin{align}\label{eq:thm3cov1}
\textup{Tr}\left(\textup{Cov}[\hat{\boldsymbol{g}}]\right)
      \leq &~ \frac{1}{n}\mathbb{E}_{\boldsymbol{u}\sim\textup{Unif}(S^{d-1}), w_+,w_-}\Bigg[ \left(\frac{d}{2}\right)^2\big(f(\boldsymbol{x}+Z\boldsymbol{u}) 
       -f(\boldsymbol{x}-Z\boldsymbol{u}) +w_--w_-\big)^2\Bigg]\nonumber\\
      = &~\frac{d^2}{4n}\mathbb{E}_{\boldsymbol{u}\sim\textup{Unif}(S^{d-1})}\left[\left(f(\boldsymbol{x}+Z\boldsymbol{u})-f(\boldsymbol{x}-Z\boldsymbol{u}) \right)^2+2\right].
\end{align}
The identity above uses the fact that additive noises are unbiased and have bounded variances.

Note that from the Lipschitz Hessian condition, we have that 
$$|f(\boldsymbol{x}\pm Z\boldsymbol{u})-f_2(\boldsymbol{x}\pm Z\boldsymbol{u})|\leq \frac{1}{6}\rho ||Z\boldsymbol{u}||_2^3  \leq \frac{1}{6}\rho \lambda_Z^3,$$
where $f_2$ is the Taylor polynomial of $f$ expanded at $\boldsymbol{x}$ up to the quadratic terms. Consequently, inequality \eqref{eq:thm3cov1} implies
\begin{align}
 \textup{Tr}\left(\textup{Cov}[\hat{\boldsymbol{g}}]\right)
      \leq &~ \frac{d^2}{4n}\mathbb{E}\left[\left(|f_2(\boldsymbol{x}+Z\boldsymbol{u})-f_2(\boldsymbol{x}-Z\boldsymbol{u})|+\frac{1}{3}\rho \lambda_Z^3 \right)^2+2\right]\nonumber\\
      = &~\frac{d^2}{4n}\mathbb{E}\left[\left(|2Z\boldsymbol{u}\cdot\nabla f(\boldsymbol{x}) |+\frac{1}{3}\rho \lambda_Z^3 \right)^2+2\right]\nonumber\\
      \leq &~ \frac{d^2}{4n}\mathbb{E}\left[2\cdot |2Z\boldsymbol{u}\cdot\nabla f(\boldsymbol{x})|^2+2\left(\frac{1}{3}\rho \lambda_Z^3 \right)^2+2\right]\nonumber\\
      = &~ \frac{2d}{n}||Z\nabla f(\boldsymbol{x})||_2^2 +\frac{d^2}{18n}\left(\rho \lambda_Z^3 \right)^2+\frac{d^2}{2n} \nonumber
\end{align}
where the expectations are taken of $\boldsymbol{u}\sim\textup{Unif}(S^{d-1})$, and the last equality is due to the well-known fact that $\mathbb{E}\left[\vct u \vct u^{\intercal}\right]=\frac{1}{d}I_d$.

\ifdefined\edit

\section{Proof Ideas for Theorem \ref{thm2}} 

To illustrate the proof idea for the lower bound, we first consider the 1D case. The proof for general $d$ can be found in Appendix \ref{app:gend}. The gist of our proof is to construct a pair of hard-instance functions that need to 
sufficiently distant from each other to avoid trivial optimizers with low simple regret. We also require them to be sufficiently close to each other so that they are indistinguishable without sufficiently many samples. These requirements are captured quantitatively in the following result, which is proved using an analysis of KL divergence. Here we assume their correctness and focus on the construction. 

\begin{definition}\label{def:sampling_error}
For any (Borel measurable) function class $\mathcal{F}_{\textup{H}}$ and any distribution $p$ defined on $\mathcal{F}_{\textup{H}}$, we define the \emph{uniform sampling error} to be
$P_{{\epsilon}}\triangleq\inf_{\boldsymbol{x}}  \mathbb{P}_{f\sim p}[f(\boldsymbol{x})-\inf f\geq \epsilon].$
We also define the \emph{maximum local variance} to be 
$V\triangleq \sup_{\boldsymbol{x}}  \textup{Var}_{f\sim p}[f(x)].$
\end{definition}

\begin{lemma}[Restatement of Proposition 7 in \cite{theotherpaper}]\label{p:lower_key}
For any sampling algorithm to achieve an expected simple regret  of $\epsilon>0$ over a function class $\mathcal{F}_{\textup{H}}$, if $P_{2\epsilon/c}\geq c$ for some universal constant $c\in(0,1)$,  and  
 the observation noises are standard Gaussian, then the required sample complexity to achieve a minimax regret of $\epsilon$ is at least $\Omega(1/V)$. 
\end{lemma}

We construct our hard instances using the following function 
\begin{align*} g(x)= \begin{cases} \frac{1}{2}\left(\sin\left(\frac{1}{2}x\right)+1\right) &~ \textup{ if } x\in \left(-{\pi}, {3\pi} \right] \\  
-\cos x-1&~ \textup{ if } x\in\left(-{3\pi},-\pi\right]\\ 0 &~ \textup{otherwise}.
\end{cases}
\end{align*}

Some key properties of $g(x)$ to be used are that its  differential $g'(x)$ is $1$-Lipschitz, and we have $|g'(x)|\leq 1$ for all $x$. 
Our hard instances consist of two functions. 
We define 
\begin{align*} f_1(x)=&~ M x^2+y_0 \int_{-{\pi}}^{x/x_0} g(z)dz, \ \ \ \ \ \ \ \ \ \ \ \ \ \ 
f_2(x)=  M x^2 + y_0 \int_{-{\pi}}^{-x/x_0} g(z)dz,
\end{align*}
where $y_0$, $x_0$ are normalization factors given by
$y_0=\frac{1}{\pi\sqrt{T}}$, 
$x_0=\left(\frac{y_0}{\rho}\right)^{\frac{1}{3}}.$
The normalization factors are chosen to satisfy the Lipschitz Hessian condition and a maximum local variance bound required for a KL-divergence based approach presented in Lemma~\ref{p:lower_key}.

Specifically, the choice of $x_0$ and the fact that $g'(x)$ is $1$-Lipschitz imply that both $f_1$ and $f_2$ satisfy the Lipschitz Hessian condition. Then because the absolute value of integration of $g(x)$ is bounded by $2\pi$, one can show that the maximum local variance for the function class $\{f_1,f_2\}$ is no greater than $\pi^2 y_0^2=\frac{1}{T}$ for the uniform prior distribution, which is to be used to show the sample complexity lower bound. 

We first check that both $f_1$ and $f_2$ are within our function class of interests. 
Note that both $f''_1(x)$ and $f''_2(x)$ belong to the interval $\left[2M-\frac{5}{4}\frac{y_0}{x_0^2}, 2M-\frac{3}{4}\frac{y_0}{x_0^2}\right]$. From the fact that  $\lim_{T\rightarrow \infty}\frac{y_0}{x_0^2}=0$ and $M>0$, we have both $f''_1(x)>M$ and $f''_2(x)>M$ for all $x$ for sufficiently large $T$. So the strong convexity requirement is satisfied. On the other hand, consider any global minimum point $x^*$ of either $f_1$ or $f_2$. Because of their differentiability, we must have $f_1'(x)=0$ or $f_2'(x)=0$. Note that for all $x$, we have $ |g(x)|\leq 2$, and
\begin{align*}
    f'_1(x)=&~2Mx+g\left(\frac{x}{x_0}\right)\frac{y_0}{x_0}  \ \ \ \ \ \ \ \ \ \ \ \ \ \  \ \ \ \ \ \ \ \ \ \ \  
    f'_2(x)=2Mx-g\left(\frac{x}{x_0}\right)\frac{y_0}{x_0}.
\end{align*}
We must have  
     $|x^*|\leq {\frac{y_0}{x_0}}/M$, 
where the RHS is $o(1)$ for large $T$. Combined with strong convexity, this inequality implies that assumption A3 holds for both functions.
To conclude, we have proved that $f_1, f_2\in\mathcal{F}(\rho,M,R)$ for sufficiently large $T$.

Now we let $\epsilon=\frac{1}{128M}\left(\frac{y_0}{x_0}\right)^2$ and $c=\frac{1}{2}$ to apply Lemma~\ref{p:lower_key}. Note that
\begin{align*}
    \liminf_{T\rightarrow \infty}  T^{\frac{2}{3}}\epsilon=\frac{\rho^{\frac{2}{3}}}{128\pi^{\frac{4}{3}} M}.
\end{align*}
The quantity $\epsilon$ exactly matches the lower bounds we aim to prove. Therefore, it remains to check that the required condition on uniform sampling errors in Definition~\ref{def:sampling_error} are satisfied. 

Formally, we need to show that 
$f_k(0)-\inf_{x} f_k(x)\geq 4\epsilon$ for $k\in\{1,2\}$, so that the uniform sampling error $P_{4\epsilon}$ under uniform distribution over $\mathcal{F}_{\textup{H}}$ is lower bounded by $\frac{1}{2}$ and Lemma~\ref{p:lower_key} can be applied. 
Without loss of generality, we focus on the case of $k=1$. Note that $f_1''(x)\leq 2M+\frac{y_0}{4x_0^2}$ for all $x\in[-\pi x_0,0]$. Therefore, we have 
\begin{align*}
    f_1(x)- f_1(0)\leq &~ f'_1(0)x +\frac{1}{2}x^2\sup_{z\in [-\pi x_0,0]} f_1''(z)
    \leq  \frac{y_0}{2x_0}x + \frac{1}{2}x^2 \left( 2M+\frac{y_0}{4x_0^2} \right)
\end{align*}
for $x\in[-\pi x_0,0]$, and $ \lim_{T\rightarrow\infty}x_0 =0 $. Consider any sufficiently large $T$ such that $\frac{y_0}{4x_0^2}\leq 2M$, we can choose $x=-\frac{y_0}{2x_0}\frac{1}{2M+\frac{y_0}{4x_0^2}} $ for the above bound, which falls into the interval of $[-\pi x_0, 0]$. Then we have 
\begin{align*}
    \inf_{x} f_1(x)\leq &~ f_1\left(-\frac{y_0}{2x_0}\frac{1}{2M+\frac{y_0}{4x_0^2}}\right) 
    \leq  f_1(0) -\frac{1}{2}\left(\frac{y_0}{2x_0}\right)^2\frac{1}{2M+\frac{y_0}{4x_0^2}}
    \leq 
    f_1(0)-4\epsilon. 
\end{align*}

We use this inequality to lower bound the minimum sampling error. Note that $f_1$ is an increasing function for $x\geq 0$ and $ \inf_{x} f_1(x)= \inf_{x} f_2(x)$. We have $f_1(x)\geq \inf_{x} f_2(x)+4\epsilon$  for $x\geq 0$. Following the same arguments, we also have $f_2(x)\geq \inf_{x} f_1(x)+4\epsilon$  for $x\leq 0$. Recall the definition of uniform sampling error in Definition~\ref{def:sampling_error}. We have essentially proved that $P_{4\epsilon}\geq \frac{1}{2}$. According to earlier discussions, this implies that the minimax simple regret is lower bounded by $\epsilon=\Omega\left( \frac{\rho^{\frac{2}{3}}T^{-\frac{2}{3}}}{ M}\right)$.

\else

\fi

\section{Conclusion and Future Work}
In this work, we achieve the first minimax simple regret for bandit optimization of second-order smooth and strongly convex functions. We derived the matching upper and lower bounds and proposed an algorithm that integrates a bootstrapping stage with a mirror-descent stage. Our key technical innovations include a sharp characterization of the spherical-sampling gradient estimator under higher-order smoothness conditions and a novel iterative method for the bootstrapping stage that remains effective with unbounded Hessians.

While these advancements settle the fundamental problem of optimizing second-order smooth and strongly convex functions with zeroth-order feedback, the techniques and insights presented in this paper also pave the way for further research in this domain. One interesting follow-up direction is to generalize our analysis to the online setting for the average regret metric. Additionally, investigating the fundamental tradeoff between simple regret and average regret could yield valuable insights for task-specific algorithmic designs.

\appendix


\bibliography{references}
\bibliographystyle{iclr2024_conference}

\section{Detailed Comparison on Different Smoothness Conditions}\label{app:comp}

In a relevant line of work, the higher-order smoothness of the objective function in the $k=3$ case is characterized by the following 
generalized H\"{o}lder condition. 
\begin{align*}
    \left|f(\vct z) - f(\vct x) -\nabla f(\vct z) (\vct z-\vct x) -\frac{1}{2} (\vct z-\vct x)^\intercal \left(\nabla^2 f(\vct z)\right) {(\vct z-\vct x)}\right| \leq L ||\vct z-\vct x||_2^3. 
\end{align*}
Note that the $\rho$-Lipschitz Hessian condition in our work implies an H\"{o}lder condition with parameter $L=\frac{\rho}{6}$. A direct application of the works in Table \ref{tab:related} requires order-wise larger sample complexities in our setting on top of the additional $k=2$ smoothness condition.  
On the other hand, the $L$-Holder condition implies $\rho=O(\sqrt{d}L)$-Lipschitz Hessian.
 Hence, a direct application of our algorithm order-wise improves the sample complexity in the setting of generalized H\"{o}lder condition in a polynomial factor of $d$ as well.  

In terms of the characterization of gradient estimators, the prior works of \cite{bach2016highly,akhavan2020exploiting,novitskii2021improved}
used isotropic sampling over a bounded set of radius $r$ and presented upper bounds on the estimation bias of $O(Lr^2)$, $O(Ldr^2)$, and $O(L\sqrt{d}r^2)$, respectively. In this special case, our Theorem \ref{thm:gdest} implies an upper bound of $O(\rho r^2/\sqrt{d})$, and similar to the above analysis, this bound strengths the bounds in prior works.

\section{Proof of Theorem \ref{thm2}} 

To illustrate the proof idea, 
we start with the case of $d=1$.

\subsection{Illustrating example: 1D case}

The gist of our proof is to construct a pair of hard-instance functions that need to 
sufficiently distant from each other to avoid trivial optimizers with low simple regret. We also require them to be sufficiently close to each other so that they are indistinguishable without sufficiently many samples. These requirements are captured quantitatively in the following result, which is proved using an analysis of KL divergence. Here we assume their correctness and focus on the construction. 

\begin{definition}\label{def:sampling_error}
For any (Borel measurable) function class $\mathcal{F}_{\textup{H}}$ and any distribution $p$ defined on $\mathcal{F}_{\textup{H}}$, we define the \emph{uniform sampling error} to be
\begin{align*}
P_{{\epsilon}}\triangleq\inf_{\boldsymbol{x}}  \mathbb{P}_{f\sim p}[f(\boldsymbol{x})-\inf f\geq \epsilon].
\end{align*}
We also define the \emph{maximum local variance} to be 
\begin{align*}
V\triangleq \sup_{\boldsymbol{x}}  \textup{Var}_{f\sim p}[f(x)].
\end{align*}
\end{definition}

\begin{lemma}[Restatement of Proposition 7 in \cite{theotherpaper}]\label{p:lower_key}
For any sampling algorithm to achieve an expected simple regret  of $\epsilon>0$ over a function class $\mathcal{F}_{\textup{H}}$, if $P_{2\epsilon/c}\geq c$ for some universal constant $c\in(0,1)$,  and  
 the observation noises are standard Gaussian, then the required sample complexity to achieve a minimax regret of $\epsilon$ is at least $\Omega(1/V)$. 
\end{lemma}

We construct our hard instances using the following function 
\begin{align*} g(x)= \begin{cases} \frac{1}{2}\left(\sin\left(\frac{1}{2}x\right)+1\right) &~ \textup{ if } x\in \left(-{\pi}, {3\pi} \right] \\  
-\cos x-1&~ \textup{ if } x\in\left(-{3\pi},-\pi\right]\\ 0 &~ \textup{otherwise}.
\end{cases}
\end{align*}

Some key properties of $g(x)$ to be used are that its  differential $g'(x)$ is $1$-Lipschitz, and we have $|g'(x)|\leq 1$ for all $x$. 
Our hard instances consist of two functions. 
We define 
\begin{align*} f_1(x)=&~ M x^2+y_0 \int_{-{\pi}}^{x/x_0} g(z)dz, \\
f_2(x)= &~ M x^2 + y_0 \int_{-{\pi}}^{-x/x_0} g(z)dz,
\end{align*}
where $y_0$, $x_0$ are normalization factors given by
$y_0=\frac{1}{\pi\sqrt{T}}$, 
$x_0=\left(\frac{y_0}{\rho}\right)^{\frac{1}{3}}.$
The normalization factors are chosen to satisfy the Lipschitz Hessian condition and a maximum local variance bound required for a KL-divergence based approach presented in Lemma~\ref{p:lower_key}.

Specifically, the choice of $x_0$ and the fact that $g'(x)$ is $1$-Lipschitz imply that both $f_1$ and $f_2$ satisfy the Lipschitz Hessian condition. Then because the absolute value of integration of $g(x)$ is bounded by $2\pi$, one can show that the maximum local variance for the function class $\{f_1,f_2\}$ is no greater than $\pi^2 y_0^2=\frac{1}{T}$ for the uniform prior distribution, which is to be used to show the sample complexity lower bound. 

We first check that both $f_1$ and $f_2$ are within our function class of interests. 
Note that both $f''_1(x)$ and $f''_2(x)$ belong to the interval $\left[2M-\frac{5}{4}\frac{y_0}{x_0^2}, 2M-\frac{3}{4}\frac{y_0}{x_0^2}\right]$. From the fact that  $\lim_{T\rightarrow \infty}\frac{y_0}{x_0^2}=0$ and $M>0$, we have both $f''_1(x)>M$ and $f''_2(x)>M$ for all $x$ for sufficiently large $T$. So the strong convexity requirement is satisfied. On the other hand, consider any global minimum point $x^*$ of either $f_1$ or $f_2$. Because of their differentiability, we must have $f_1'(x)=0$ or $f_2'(x)=0$. Note that for all $x$, we have $ |g(x)|\leq 2$, and
\begin{align*}
    f'_1(x)=&~2Mx+g\left(\frac{x}{x_0}\right)\frac{y_0}{x_0}\\
    f'_2(x)=&~2Mx-g\left(\frac{x}{x_0}\right)\frac{y_0}{x_0}.
\end{align*}
We must have  
     $|x^*|\leq {\frac{y_0}{x_0}}/M$, 
where the RHS is $o(1)$ for large $T$. Combined with strong convexity, this inequality implies that assumption A3 holds for both functions.
To conclude, we have proved that $f_1, f_2\in\mathcal{F}(\rho,M,R)$ for sufficiently large $T$.

Now we let $\epsilon=\frac{1}{128M}\left(\frac{y_0}{x_0}\right)^2$ and $c=\frac{1}{2}$ to apply Lemma~\ref{p:lower_key}. Note that
\begin{align*}
    \liminf_{T\rightarrow \infty}  T^{\frac{2}{3}}\epsilon=\frac{\rho^{\frac{2}{3}}}{128\pi^{\frac{4}{3}} M}.
\end{align*}
The quantity $\epsilon$ exactly matches the lower bounds we aim to prove. Therefore, it remains to check that the required condition on uniform sampling errors in Definition~\ref{def:sampling_error} are satisfied. 

Formally, we need to show that 
$f_k(0)-\inf_{x} f_k(x)\geq 4\epsilon$ for $k\in\{1,2\}$, so that the uniform sampling error $P_{4\epsilon}$ under uniform distribution over $\mathcal{F}_{\textup{H}}$ is lower bounded by $\frac{1}{2}$ and Lemma~\ref{p:lower_key} can be applied. 
Without loss of generality, we focus on the case of $k=1$. Note that $f_1''(x)\leq 2M+\frac{y_0}{4x_0^2}$ for all $x\in[-\pi x_0,0]$. Therefore, we have 
\begin{align*}
    f_1(x)- f_1(0)\leq &~ f'_1(0)x +\frac{1}{2}x^2\sup_{z\in [-\pi x_0,0]} f_1''(z)\\
    \leq &~ \frac{y_0}{2x_0}x + \frac{1}{2}x^2 \left( 2M+\frac{y_0}{4x_0^2} \right)
\end{align*}
for $x\in[-\pi x_0,0]$, and $ \lim_{T\rightarrow\infty}x_0 =0 $. Consider any sufficiently large $T$ such that $\frac{y_0}{4x_0^2}\leq 2M$, we can choose $x=-\frac{y_0}{2x_0}\frac{1}{2M+\frac{y_0}{4x_0^2}} $ for the above bound, which falls into the interval of $[-\pi x_0, 0]$. Then we have 
\begin{align*}
    \inf_{x} f_1(x)\leq &~ f_1\left(-\frac{y_0}{2x_0}\frac{1}{2M+\frac{y_0}{4x_0^2}}\right) \\
    \leq &~ f_1(0) -\frac{1}{2}\left(\frac{y_0}{2x_0}\right)^2\frac{1}{2M+\frac{y_0}{4x_0^2}}\\
    \leq &~
    f_1(0)-4\epsilon. 
\end{align*}

We use this inequality to lower bound the minimum sampling error. Note that $f_1$ is an increasing function for $x\geq 0$ and $ \inf_{x} f_1(x)= \inf_{x} f_2(x)$. We have $f_1(x)\geq \inf_{x} f_2(x)+4\epsilon$  for $x\geq 0$. Following the same arguments, we also have $f_2(x)\geq \inf_{x} f_1(x)+4\epsilon$  for $x\leq 0$. Recall the definition of uniform sampling error in Definition~\ref{def:sampling_error}. We have essentially proved that $P_{4\epsilon}\geq \frac{1}{2}$. According to earlier discussions, this implies that the minimax simple regret is lower bounded by $\epsilon=\Omega\left( \frac{\rho^{\frac{2}{3}}T^{-\frac{2}{3}}}{ M}\right)$.

 \subsection{
 Proof for the General Case}\label{app:gend}

The generalization of the earlier 1D lower bound is obtained by constructing a set of hard-instance functions where the optimization problem over this subset consists of $d$ binary hypothesis estimation problems, each identical to a 1D construction. Formally, for any $\vct s=(s_1,s_2,...,s_d)\in\{1,2\}^d$ and any input $\vct x=(x_{(1)},x_{(2)},...,x_{(d)})$, we let 
\begin{align*} 
f_{\vct s}(\vct x)&= \sum_{j=1}^{d} f_{s_j}(x_{(j)}).
\end{align*}
One can verify that $f_{\vct s}\in \mathcal{F}(\rho,M,R)$ for all $\vct s$ for sufficiently large $T$.  

Note that the simple regret for the above function class can be written as the sum of $d$ individual terms $\sum_{j=1}^{d} \left(f_{s_j}(x_{(j)})-\inf_x f_{s_j}(x)\right)$. As proved earlier, the expectation of each term associated with any index $j$ is at least $\Omega\left( \frac{\rho^{\frac{2}{3}}T^{-\frac{2}{3}}}{ M}\right)$ even if all entries of $\vct s$ except $s_j$ is known. Therefore, the total expected regret is lower bounded by $\Omega\left( \frac{d\rho^{\frac{2}{3}}T^{-\frac{2}{3}}}{ M}\right)$ .  




\section{Proof of Theorem \ref{thm:pgest}}\label{sec:pthest}
We use the following elementary facts, which are versions of well-known properties of subgaussian and subexponential distributions in \cite{vershynin2018high}, but with explicit and possibly improved constant factors. For completeness, we provide their proofs in Appendix \ref{app:ppa_ele}. 
\begin{proposition}\label{prop:pa_ele}
For any real-valued zero-mean independent random variables $ z_1,..., z_k$, 
if
\begin{align}\label{eq:appd_subg}
\mathbb{P}[|z_j|\geq K]\leq 2\exp\left({-\frac{K^2}{\sigma_j^2}}\right) && \forall\, j \in [k], K\in [0,+\infty),
\end{align}
for some $\sigma_1,...,\sigma_k$, then
\begin{align}\label{eq:appd_res}
\mathbb{P}\left[\left|\sum_{j=1}^{k} z_j\right|\geq K\right]\leq 2\exp\left({-\frac{K^2}{4\sum_{j=1}^{k}\sigma_j^2}}\right) && \forall\, K\in [0,+\infty).
\end{align}
\end{proposition}

\begin{proposition}\label{prop:pa_ele2}
For any real-valued independent random variables $ z_1,..., z_k$, 
if
\begin{align}\label{eq:appd_sube}
\mathbb{P}[|z_j|\geq K]\leq 2\exp\left({-\frac{K}{\sigma_j}}\right) && \forall\, j \in [k], K\in [0,+\infty),
\end{align}
for some positive $\sigma_1,...,\sigma_k$, then
\begin{align}\label{eq:appd_res2}
\mathbb{P}\left[\left|\sum_{j=1}^{k} z_j\right|\geq K\right]\leq 2\exp\left({-\frac{K}{3\sum_{j=1}^{k}\sigma_j}}\right) && \forall\, K\in [0,+\infty).
\end{align}
\end{proposition}

\begin{proof}[Proof of equation \eqref{eq431}.]
 We first prove the bound entry-vise. Consider any $m_k$, which contains a summation of $2n$ independent subgaussian variables. By Prop. \ref{prop:pa_ele} we have that
\begin{align}
    \mathbb{P}\left[|m_k-\mathbb{E}\left[m_k\right]|\geq K\right]& \leq  2\exp\left({-\frac{K^2}{2/nr^2}}\right) && \forall\, K\in [0,+\infty). 
\end{align} 
Then, by the Lipschitz Hessian condition, the bias for each entry is bounded as follows. 
\begin{align}
   \left| \mathbb{E}\left[m_k\right]-\frac{\partial }{\partial x_k} f(\vct x)\right|\leq \frac{1}{6}\rho r^2, 
\end{align} 
where $x_k$ denotes the $k$th entry of $\vct x$.
Hence, each $\left|m_k-\frac{\partial }{\partial x_k} f(\vct x)\right|^2$ is subexpoential, i.e., 
\begin{align}
    \mathbb{P}\left[\left|m_k-\frac{\partial }{\partial x_k} f(\vct x)\right|^2\geq K^2\right]& \leq \mathbb{P}\left[\left|m_k-\mathbb{E}\left[m_k\right]\right|\geq K-  \left| \mathbb{E}\left[m_k\right]-\frac{\partial }{\partial x_k} f(\vct x)\right|\right] \nonumber\\&\leq 
 \mathbb{P}\left[\left|m_k-\mathbb{E}\left[m_k\right]\right|\geq K-  \frac{1}{6}\rho r^2\right]\nonumber\\ &\leq \max\left\{ 2\exp\left({-\frac{\left(\max\{K-  \frac{1}{6}\rho r^2,0\}\right)^2}{2/nr^2}}\right),1\right\}\nonumber\\ &\leq 2\exp\left({-\frac{K^2}{\frac{\rho^2r^4}{4}+\frac{4}{nr^2}}}\right). 
\end{align} 
By the independence of $m_k$'s, we can apply Prop. \ref{prop:pa_ele2} to the inequality. Therefore,  
\begin{align}
    \mathbb{P}\left[||\hat{\vct m}-\nabla f(\vct x)||_2\geq K\right]&=  \mathbb{P}\left[\left| \sum_{k=1}^{d} |m_k-\frac{\partial }{\partial x_k} f(\vct x)|^2\right| \geq K^2\right]\nonumber\\& \leq 2\exp\left({-\frac{K^2}{\frac{3d\rho^2r^4}{4}+\frac{12d}{nr^2}}}\right) && \forall\, K\in [0,+\infty). 
\end{align} 

\end{proof}

\begin{proof}[Proof of equation \eqref{eq432}.]
We first provide the entry-wise bounds for the intermediate estimator $\hat{H}_0$. Each diagonal entry ${H}_{kk}$ contains the weighted average of $3n$ subgaussian variables. Conditioned on any realization of $y$, which is shared among all diagonal elements, Prop. \ref{prop:pa_ele} can be applied for the rest of the $2n$ terms, and provides the following bounds.
\begin{align}
\mathbb{P}[  | {H}_{kk}-\mathbb{E}[H_{kk}|y]  | \geq K ] 
\leq   2\exp\left({-\frac{K^2}{8/nr^4}}\right) && \forall\, K\in [0,+\infty). 
\end{align}
Then, because the off-diagonal entries are independent, we have the following bounds for any $j\neq k$. 
\begin{align}
\mathbb{P}[  | {H}_{jk}-\mathbb{E}[H_{jk}]  | \geq K ]\leq   2\exp\left({-\frac{K^2}{1/nr^4}}\right) && \forall\, K\in [0,+\infty). 
\end{align}
Hence, similar to the earlier proof steps, Prop. \ref{prop:pa_ele2} implies that 
\begin{align}
    \mathbb{P}\left[||\hat{H}_0-\mathbb{E}[\hat{H}_0|y]||_{\textup{F}}^2\geq K^2 
    \right]&\leq 2\exp\left({-\frac{K^2}{6d(d+3)/nr^4}}\right)\nonumber\\&\leq 2\exp\left({-\frac{K^2}{24d^2/nr^4}}\right) .\label{ineq:pat1} 
\end{align} 

Now, we take into account the estimation bias and the error of $y$. By Lipschitz Hessian, it is clear that 
\begin{align}
\left|H_{jk}-\frac{\partial}{\partial x_j}\frac{\partial}{\partial x_k} f(\vct x)\right|\leq &\begin{cases}\frac{1}{3}\rho r &\textup{if } j=k, \\ \frac{\sqrt{2}}{3}\rho r &\textup{otherwise.} \end{cases}\nonumber
\end{align}
Hence,
\begin{align}
\left|\left|\hat{H}_0-\nabla^2 f(\vct x)\right|\right|_{\textup{F}}
\leq \frac{\sqrt{2}d\rho r}{3}. \label{ineq:pat2} 
\end{align}
Furthermore, note that $\mathbb{E}[\hat{H}_0]-\mathbb{E}[\hat{H}_0|y]= {2(y-f(\vct x))}I_d/r^2$, where $I_d$ denotes the identity matrix. The subgaussian condition and Prop. \ref{prop:pa_ele} imply that 
\begin{align}
\mathbb{P}[  || \mathbb{E}[\hat{H}_0]-\mathbb{E}[\hat{H}_0|y]  ||_{\textup{F}} \geq K   ] &= \mathbb{P}\left[  \left| y- f(\vct x)  \right| \geq \frac{Kr^2}{2\sqrt{d}}  \right] \nonumber\\ &\leq   2\exp\left({-\frac{K^2}{8d/nr^4}}\right) && \forall\, K\in [0,+\infty).\label{ineq:pat3} 
\end{align}
We can combine the above bounds using triangle inequality and the union bound. Specifically, from inequalities \eqref{ineq:pat1}, \eqref{ineq:pat2}, and \eqref{ineq:pat3}, we have the following bound for any $K\geq \frac{\sqrt{2}d\rho r}{3}$. 
\begin{align}
    \mathbb{P}\left[||\hat{H}_0-\nabla^2 f(\vct x )||_{\textup{F}}\geq K\right]&\leq \mathbb{P}\left[||\hat{H}_0-\mathbb{E}[\hat{H}_0]||_{\textup{F}}\geq \frac{2}{3}\left(K- \frac{\sqrt{2}d\rho r}{3}\right)\right] \nonumber\\ &\leq \mathbb{P}\left[||\hat{H}_0-\mathbb{E}[\hat{H}_0|y]||_{\textup{F}}\geq \frac{1}{3}\left(K- \frac{\sqrt{2}d\rho r}{3}\right)
    \right]
    \nonumber
    \\
    & \ +  \mathbb{P}\left[  || \mathbb{E}[\hat{H}_0]-\mathbb{E}[\hat{H}_0|y]  ||_{\textup{F}} \geq \left(K- \frac{\sqrt{2}d\rho r}{3}\right)   \right]\nonumber
    \\
    &\leq 2\exp\left({-\frac{\left(K- \frac{\sqrt{2}d\rho r}{3}\right)^2}{54d^2/nr^4}}\right)+ 2\exp\left({-\frac{\left(K- \frac{\sqrt{2}d\rho r}{3}\right)^2}{72d/nr^4}}\right)\nonumber. 
\end{align}
Utilize the fact that any probability measure is no greater than $1$, the above inequality implies that 
\begin{align}
    \mathbb{P}\left[||\hat{H}_0-\nabla^2 f(\vct x )||_{\textup{F}}\geq K\right]&\leq 2\exp\left({-\frac{\left(\max\left\{K- \frac{\sqrt{2}d\rho r}{3},0\right\}\right)^2}{{128d^2}/{nr^4}}}\right) \nonumber \\ &\leq 2\exp\left({-\frac{K^2}{2{d^2\rho^2 r^2}+\frac{144d^2}{nr^4}}}\right) && \forall K\in[0,+\infty).
\end{align}
Finally, the needed bound for $\hat{H}$ is due to the projection to a convex set where the target $\nabla^2 f(\vct x)$ belongs. Hence, the distance is not increased w.p.1, i.e., we always have $||\hat{H}-\nabla^2 f(\vct x )||_{\textup{F}}\leq ||\hat{H}_0-\nabla^2 f(\vct x )||_{\textup{F}}$.
\end{proof}

 \section{Proof of Proposition \ref{pp:ptech1}}\label{app:pp_ptech1}
Inequality \eqref{ineq:tech1eq6} is derived from the following approximations, which are due to the Lipschitz Hessian condition at $\vct x_{\textup{B}}$. 
\begin{align} 
f(\vct x)\leq& \ f(\vct x_{\textup{B}}) 
+ \tilde{f}(\vct x)-\tilde{f}(\vct x_{\textup{B}}) +\frac{1}{6} \rho ||\vct x- \vct x_{\textup{B}}||_2^3,\label{ineq:ptfirst}\\
f(\vct x^*)\geq& \ f(\vct x_{\textup{B}}) 
+ \tilde{f}(\vct x^*)-\tilde{f}(\vct x_{\textup{B}}) -\frac{1}{6} \rho ||\vct x^*- \vct x_{\textup{B}}||_2^3.\label{ineq:ptfirst2}
\end{align}
Noting that $\tilde{f}(\vct x^*)\geq 0$, the above inequalities imply that  
\begin{align}\label{ineq:ptsecx} 
f(\vct x)- f(\vct x^*)\leq 
  \tilde{f}(\vct x)+\frac{1}{6} \rho ||\vct x- \vct x_{\textup{B}}||_2^3 +\frac{1}{6} \rho ||\vct x^*- \vct x_{\textup{B}}||_2^3.
\end{align}
By strong convexity, we have $||\vct x^*- \vct x_{\textup{B}}||_2^2\leq \frac{2(f(\vct x_{\textup{B}})- f(\vct x^*))}{M}$. Hence, it remains to provide an upper bound for $\frac{1}{6} \rho ||\vct x- \vct x_{\textup{B}}||_2^3$. 

When $||\vct x- \vct x_{\textup{B}}||_2\leq \sqrt{3}||\vct x- \tilde{\vct x}||_2$, we apply the condition $||\vct x- \vct x_{\textup{B}}||_2\leq \frac{M}{\rho}$ to obtain that 
\begin{align*} 
\frac{1}{6} \rho ||\vct x- \vct x_{\textup{B}}||_2^3 &\leq \frac{1}{2} M ||\vct x- \tilde{\vct x}||_2^2 \leq \tilde{f}(\vct x),
\end{align*}
where the last step is due to the strong convexity of $\tilde{f}$.
This implies inequality \eqref{ineq:tech1eq6}. 

For the other case, we have $||\vct x- \vct x_{\textup{B}}||_2\geq \sqrt{3}||\vct x- \tilde{\vct x}||_2$. 
We replace the variable $\vct x$ in inequality \eqref{ineq:ptfirst} with $\tilde{\vct x}$ to obtain that 
\begin{align}
f(\vct x^*)\leq f(\tilde{\vct x})\leq& \ f(\vct x_{\textup{B}}) 
-\tilde{f}(\vct x_{\textup{B}}) +\frac{1}{6} \rho ||\tilde{\vct x}- \vct x_{\textup{B}}||_2^3. \label{ineq:ptfirst3}
\end{align}
By strong convexity, 
\begin{align}
\tilde{f}(\vct x_{\textup{B}})\geq  \frac{1}{2} M ||\tilde{\vct x}- \vct x_{\textup{B}}||_2^2, \label{ineq:ptfirst3fl}
\end{align}
and by triangle inequality, we have that
\begin{align*} 
||\tilde{\vct x}- \vct x_{\textup{B}}||_2\leq  ||\vct x- \vct x_{\textup{B}}||_2+||\vct x- \tilde{\vct x}||_2 \leq \left(1+\frac{1}{\sqrt{3}}\right)\frac{M}{\rho}. 
\end{align*}
Hence, to summarize, 
\begin{align*}
f(\vct x_{\textup{B}}) - f(\vct x^*)\geq  \left(
\frac{1}{3}-\frac{{1}}{6\sqrt{3}}\right) M ||\tilde{\vct x}- \vct x_{\textup{B}}||_2^2,
\end{align*}
and the needed result is obtained by applying the above to inequality \eqref{ineq:ptsecx}. 

Now we prove inequality \eqref{ineq:tech1eq7}. The proof consists of three steps. For brevity, let $H\triangleq \nabla^2 f( \vct x_{\textup{B}})$ and $\vct x_{+}=\vct x_{\textup{B}}-\hat{H}^{-1}Z^{-1}\hat{\vct g}$. We first prove that 
\begin{align}\label{ineq:tech8_p15}
\tilde{f}(\vct x)\leq \tilde{f}(\vct x_+) +  \mathbbm{1}\left(\tilde{f}(\vct x_{\textup{B}})\geq  \frac{M^3}{8\rho^2}\right) \cdot \tilde{f}(\vct x_{\textup{B}}).  
\end{align}
We shall repetitively use the fact that $|| \vct z- \tilde{\vct x} ||_2\leq \sqrt{2\tilde{f}(\vct z)/M}$ for any $\vct z\in \mathbb{R}^d$, which is due to strong convexity. When both $\tilde{f}(\vct x_+)$ and $\tilde{f}(\vct x_{\textup{B}})$ are no greater than $\frac{M^3}{8\rho^2}$, both $||\vct x_+-\tilde{\vct x}||_2$ and $||\vct x_{\textup{B}}-\tilde{\vct x}||_2$ are no greater than $\frac{M}{2\rho}$. By triangle inequality, we have $||\vct x_+-\vct x_{\textup{B}}||_2\leq \frac{M}{\rho}$. Recall the construction of $\vct x$, which is identical to $\vct x_+$ in this case, inequality \eqref{ineq:tech8_p15} clearly holds. 
Otherwise, note that $\vct x$ belongs to the line segment between $\vct x_{\textup{B}}$ and $\vct x_+$. By convexity, we always have $\tilde{f}(\vct x)\leq \max\{\tilde{f}(\vct x_+), \tilde{f}(\vct x_{\textup{B}})\}.$ Recall that in this case, $\tilde{f}(\vct x_{\textup{B}})\geq \tilde{f}(\vct x_+)$ can only hold when $\tilde{f}(\vct x_{\textup{B}})\geq  \frac{M^3}{8\rho^2}$, we have $\tilde{f}(\vct x)\leq \mathbbm{1}\left(\tilde{f}(\vct x_{\textup{B}})\leq  \frac{M^3}{8\rho^2}\right) \tilde{f}(\vct x_+) +  \mathbbm{1}\left(\tilde{f}(\vct x_{\textup{B}})\geq  \frac{M^3}{8\rho^2}\right)\max\{\tilde{f}(\vct x_+), \tilde{f}(\vct x_{\textup{B}})\}$, which implies inequality \eqref{ineq:tech8_p15}. 

As the second step, we prove that 
\begin{align}\label{ineq:tech1eq7p}
\mathbb{E}\left[\tilde{f}(\vct x_+)\ |\ \vct x_{\textup{B}}\right]\leq 
  \left(\frac{7d^2\rho^{\frac{4}{3}}}{M^2n_{\textup{H}}^{\frac{1}{3}}}+\frac{26d}{n_{\textup{g}}}\right) \tilde{f}(\boldsymbol{x}_{\textup{B}}) +\frac{3d\rho^{\frac{2}{3}}}{Mn_{\textup{g}}^{\frac{2}{3}}}.
\end{align}
Note that the estimation error of $\tilde{\vct x}$ can be decomposed into two terms, i.e.,
 \begin{align*}
    \vct x_+-\tilde{\vct x} &= H^{-1} \nabla f( \vct x_{\textup{B}}) - \hat{H}^{-1}Z^{-1}\hat{\vct g}\\ 
    &= (H^{-1}  - \hat{H}^{-1})\nabla f( \vct x_{\textup{B}}) + \hat{H}^{-1}Z^{-1}(Z\nabla f( \vct x_{\textup{B}})-\hat{\vct g}),
\end{align*}
where the first term is due to the error of the Hessian estimator, and the second is mostly contributed by the GradientEst estimator. 
We apply the AM-QM inequality to their quadratic forms, i.e.,
\begin{align}
\tilde{f}(\vct x_+)&=\left|\left|H^{\frac{1}{2}}\left((H^{-1}  - \hat{H}^{-1})\nabla f( \vct x_{\textup{B}}) + \hat{H}^{-1}Z^{-1}(Z\nabla f( \vct x_{\textup{B}})-\hat{\vct g})\right)\right|\right|_2^2\nonumber\\&\leq  || H^{\frac{1}{2}} (H^{-1}  - \hat{H}^{-1})\nabla f( \vct x_{\textup{B}})||_2^2+|| H^{\frac{1}{2}} \hat{H}^{-1}Z^{-1}(Z\nabla f( \vct x_{\textup{B}})-\hat{\vct g})||_2^2. \nonumber\\
&= || H^{\frac{1}{2}} (H^{-1}  - \hat{H}^{-1})\nabla f( \vct x_{\textup{B}})||_2^2+ \left(\frac{\lambda_{Z_H}}{r_{\textup{g}}}\right)^2||H^{\frac{1}{2}} \hat{H}^{-\frac{1}{2}}||^2\cdot||Z\nabla f( \vct x_{\textup{B}})-\hat{\vct g}||_2^2, \nonumber 
\end{align}
where $\lambda_{Z_H}$, $r_\textup{g}$ are defined in Algorithm \ref{alg:1gd3_opt} and $||\cdot||$ denotes the spectrum norm. 
By theorem \ref{thm:gdest}, we can first take the expectation of the above bound conditioned on any realization of $\hat{H}$. Specifically, 
\begin{align*}
\mathbb{E}\left[||Z\nabla f( \vct x_{\textup{B}})-\hat{\vct g}||_2^2 \ |\  \hat{H}, \vct x_{\textup{B}}  \right]\leq &\left|\left|Z\nabla f( \vct x_{\textup{B}})-\mathbb{E}\left[ \hat{\vct g} 
 \ |\  \hat{H}, \vct x_{\textup{B}} \right]\right|\right|_2^2+\textup{Tr}\left(\textup{Cov}\left[\hat{\vct g}\ |\  \hat{H}, \vct x_{\textup{B}} \right]\right)\\
 \leq &\left(\frac{r_{\textup{g}}^3\rho \sqrt{d}}{2(d+2)}\right)^2+
\frac{2d}{n_{\textup{g}}}||Z\nabla f(\boldsymbol{x}_{\textup{B}})||_2^2 +\frac{d^2}{18n_{\textup{g}}}\left(\rho r_{\textup{g}}^3 \right)^2+\frac{d^2}{2n_{\textup{g}}}.
\end{align*}
Recall the definition of $Z$, 
we have 
\begin{align*}
||Z\nabla f(\boldsymbol{x}_{\textup{B}})||_2^2&=  \frac{{r_{\textup{g}}^2}}{\lambda_{Z_H}^2}|| \hat{H}^{-\frac{1}{2}}\nabla f(\boldsymbol{x}_{\textup{B}})||_2^2. 
\end{align*}
Hence, by our choice of $r_{\textup{g}}$ in Algorithm \ref{alg:1gd3_opt} and note that $\lambda_{Z_H}\leq {M^{-\frac{1}{2}}}$ is implied by strong convexity, 
\begin{align*}
\mathbb{E}\left[\tilde{f}(\vct x_+) \ |\  \hat{H}, \vct x_{\textup{B}}  \right]\leq 
 & || H^{\frac{1}{2}} (H^{-1}  - \hat{H}^{-1})\nabla f( \vct x_{\textup{B}})||_2^2\\&+ ||H^{\frac{1}{2}} \hat{H}^{-\frac{1}{2}}||^2 \left(\frac{3d\rho^{\frac{2}{3}}}{4Mn_{\textup{g}}^{\frac{2}{3}}} \left(1+
\frac{2d^3}{27n_{\textup{g}}}\right)+\frac{2d}{n_{\textup{g}}}|| \hat{H}^{-\frac{1}{2}}\nabla f(\boldsymbol{x}_{\textup{B}})||_2^2\right )\\
\leq 
 & \left(|| H^{\frac{1}{2}} (H^{-1}  - \hat{H}^{-1})H^{\frac{1}{2}}||^2+\frac{2d}{n_{\textup{g}}}||H^{\frac{1}{2}} \hat{H}^{-\frac{1}{2}}||^4\right) \cdot || {H}^{-\frac{1}{2}}\nabla f(\boldsymbol{x}_{\textup{B}})||_2^2 \\&+ ||H^{\frac{1}{2}} \hat{H}^{-\frac{1}{2}}||^2 \left(\frac{3d\rho^{\frac{2}{3}}}{4Mn_{\textup{g}}^{\frac{2}{3}}} \left(1+
\frac{2d^3}{27n_{\textup{g}}}\right)\right). 
\end{align*}

To characterize the above bound, we first note that the singular values of $H^{\frac{1}{2}} \hat{H}^{-\frac{1}{2}}$ equals the eigenvalues of $\hat{H}^{-\frac{1}{2}}H\hat{H}^{-\frac{1}{2}}=I_d+\hat{H}^{-\frac{1}{2}}(H-\hat{H})\hat{H}^{-\frac{1}{2}}$. As the eigenvalues of $\hat{H}$ are no less than $M$, by triangle inequality, all eigenvalues of $(I_d+\hat{H}^{-\frac{1}{2}}(H-\hat{H})\hat{H}^{-\frac{1}{2}})$ are bounded within $[1-\frac{||H - \hat{H}||_{\textup{F}}}{M}, 1+\frac{||H - \hat{H}||_{\textup{F}}}{M}]$. 
Hence, we have
\begin{align*}
||H^{\frac{1}{2}} \hat{H}^{-\frac{1}{2}}||^2\leq 1+ \frac{||H - \hat{H}||_{\textup{F}}}{M}. 
\end{align*}
Similarly, bounds on the singular values of $H^{\frac{1}{2}} \hat{H}^{-\frac{1}{2}}$ imply bounds on the eigenvalues of $H^{\frac{1}{2}}\hat{H}^{-1}H^{\frac{1}{2}}$, i.e.,
\begin{align*}
|| H^{\frac{1}{2}} (H^{-1}  - \hat{H}^{-1})H^{\frac{1}{2}}||\leq \frac{||H - \hat{H}||_{\textup{F}}}{M}. 
\end{align*}
Therefore, we have
\begin{align*}
\mathbb{E}\left[\tilde{f}(\vct x_+) \ |\  \hat{H}, \vct x_{\textup{B}}  \right]
\leq 
 & \left(\left(\frac{||H - \hat{H}||_{\textup{F}}}{M}\right)^2+\frac{2d}{n_{\textup{g}}}\left(1+ \frac{||H - \hat{H}||_{\textup{F}}}{M}\right)^2\right) \cdot || {H}^{-\frac{1}{2}}\nabla f(\boldsymbol{x}_{\textup{B}})||_2^2 \\&+ \left(1+ \frac{||H - \hat{H}||_{\textup{F}}}{M}\right) \cdot \left(\frac{3d\rho^{\frac{2}{3}}}{4Mn_{\textup{g}}^{\frac{2}{3}}} \left(1+
\frac{2d^3}{27n_{\textup{g}}}\right)\right). 
\end{align*}
Now that the above bound is simply a polynomial of $||H - \hat{H}||_{\textup{F}}$. We can use 
Theorem \ref{thm:hest} to obtain  
$\mathbb{P}\left[\left|\left|\hat{H}-H\right|\right|_{\textup{F}}\geq  K \right] \leq   2\exp\left({-\frac{K^2}{16{d^2\rho^{\frac{4}{3}}} /n_{\textup{H}}^{\frac{1}{3}}}}\right)$ then apply a direct integration. We utilize the assumptions in the statement of proposition to obtain a simpler estimate, expressed as follows. 
\begin{align*}
\mathbb{E}\left[\tilde{f}(\vct x_+) \ |\vct x_{\textup{B}}  \right]
\leq&  
  \left(\frac{7d^2\rho^{\frac{4}{3}}}{M^2n_{\textup{H}}^{\frac{1}{3}}}+\frac{26d}{n_{\textup{g}}}\right) \cdot || {H}^{-\frac{1}{2}}\nabla f(\boldsymbol{x}_{\textup{B}})||_2^2 +\frac{3d\rho^{\frac{2}{3}}}{Mn_{\textup{g}}^{\frac{2}{3}}}. 
\end{align*}
Then, 
inequality \eqref{ineq:tech1eq7} is implied by the definition of 
$\tilde{f}$.  

For the third step, we observe that the earlier proof steps imply that  
\begin{align}\label{ineq:pp4appi}
\mathbb{E}\left[\tilde{f}(\vct x) \ |\vct x_{\textup{B}}  \right]
\leq&  
  \left(\frac{7d^2\rho^{\frac{4}{3}}}{M^2n_{\textup{H}}^{\frac{1}{3}}}+\frac{26d}{n_{\textup{g}}}+\mathbbm{1}\left(\tilde{f}(\vct x_{\textup{B}})\geq  \frac{M^3}{8\rho^2}\right)\right) \cdot \tilde{f}(\vct x_{\textup{B}}) +\frac{3d\rho^{\frac{2}{3}}}{Mn_{\textup{g}}^{\frac{2}{3}}}, 
\end{align}
and it remains characterize $\tilde{f}(\vct x_{\textup{B}})$. 
To that end, we reuse inequality \eqref{ineq:ptfirst3} and \eqref{ineq:ptfirst3fl}, which implies that $\tilde{f}(\vct x_{\textup{B}})\leq 2(f(\vct x_{\textup{B}})-f(\vct x^*))$ when $||\tilde{\vct x}- \vct x_{\textup{B}}||_2\leq \frac{3M}{2\rho}$.
For the other case, we have $||\tilde{\vct x}- \vct x_{\textup{B}}||_2\geq \frac{3M}{2\rho}$. We instead let $\vct x_{\textup{r}}\triangleq \vct x_{\textup{B}}+\sqrt{\frac{3M}{2\rho ||\tilde{\vct x}-\vct x_{\textup{B}}||_2}}(\tilde{\vct x}-\vct x_{\textup{B}})$ and the Lipschitz Hessian condition implies that 
\begin{align}
f(\vct x^*)\leq f(\tilde{\vct x}_{\textup{r}})\leq & \ f(\vct x_{\textup{B}}) 
-\tilde{f}(\vct x_{\textup{B}})+ \tilde{f}(\vct x_{\textup{r} }) +\frac{1}{6} \rho ||\tilde{\vct x}- \vct x_{\textup{r}}||_2^3. \nonumber
\end{align}
By convexity, we have $\tilde{f}(\vct x_{\textup{B}})-\tilde{f}(\vct x_{\textup{r} })\geq \sqrt{\frac{3M}{2\rho ||\tilde{\vct x}-\vct x_{\textup{B}}||_2}}\tilde{f}(\vct x_{\textup{B}})$, which can be applied to the above bound. Then, together with inequality \eqref{ineq:ptfirst3fl} and the condition of $||\tilde{\vct x}- \vct x_{\textup{B}}||_2$,  we have
\begin{align}
f(\vct x_{\textup{B}})-f(\vct x^*)\geq & \  
\sqrt{\frac{3M}{2\rho ||\tilde{\vct x}-\vct x_{\textup{B}}||_2}}\left(\tilde{f}(\vct x_{\textup{B}}) -\frac{M}{4} ||\tilde{\vct x}- \vct x_{\textup{B}}||_2^2\right)\\
\geq &\ \frac{1}{2} \left(\frac{3M}{2\rho ||\tilde{\vct x}-\vct x_{\textup{B}}||_2}\right)^{\frac{2}{3}}\tilde{f}(\vct x_{\textup{B}}) \nonumber
\\
\geq &\ \frac{3^{\frac{2}{3}}M}{4\rho^{\frac{2}{3}}} \tilde{f}(\vct x_{\textup{B}})^{\frac{2}{3}}. \nonumber
\end{align}
To summarize, the following inequality holds in both cases.
\begin{align}
\tilde{f}(\vct x_{\textup{B}})\leq  \max\left\{2(f(\vct x_{\textup{B}})-f(\vct x^*)), \frac{8\rho}{3M^{\frac{3}{2}}} (f(\vct x_{\textup{B}})-f(\vct x^*))^{\frac{3}{2}}\right\}.  
\end{align}

To apply inequality \eqref{ineq:pp4appi}, we use the following implications.
\begin{align*}
\tilde{f}(\vct x_{\textup{B}})\leq  2(f(\vct x_{\textup{B}})-f(\vct x^*))&+ \frac{8\rho}{3M^{\frac{3}{2}}} (f(\vct x_{\textup{B}})-f(\vct x^*))^{\frac{3}{2}},\\
\mathbbm{1}\left(\tilde{f}(\vct x_{\textup{B}})\geq  \frac{M^3}{8\rho^2}\right)\tilde{f}(\vct x_{\textup{B}})&\leq   \frac{8\rho}{M^{\frac{3}{2}}} (f(\vct x_{\textup{B}})-f(\vct x^*))^{\frac{3}{2}}.  
\end{align*}
Then, the derived inequality can be simplified using our assumptions on $n_{\textup{g}}$ and $n_{\textup{H}}$.

 \section{Remaining details for Theorem \ref{thm1}}\label{app:gup}

To complete the proof, we essentially need to prove inequality \eqref{eq:fststbd}. To illustrate the main ideas, we start with an analysis in a simplified setting where estimation errors for the BootstrapingEst and HessianEst functions are zero. Then, we show how the proof steps can be modified to have the errors and uncertainties incorporated.

\subsection{Analysis for the zero-error case}

We prove that when the estimation errors are set to zero, 
the first stage of Algorithm \ref{alg:1gd3_opt} reduces 
$\nabla f(\vct z_t)$ to a vector of bounded length in boundedly many iterations. This is summarized in the following proposition.
\begin{proposition}\label{prop:app_zeroerr}
For any fixed parameter values $\rho, M$, $R$, let $\left\{\vct z_t\right\}_{t\in\mathbb{N}_+}$ be sequences defined for any $f\in\mathcal{F}(\rho, M$, $R)$, such that $\vct z_1=\vct 0$ and $\left(\vct z_{t+1}-\vct z_{t}\right)$ equals $-\tilde{H}_{t}^{-1}\nabla f(\vct z_t)$, where $\tilde{H}_t$ is a matrix that has the same eigenvectors of $\nabla^2 f(\vct z_t)$, with each eigenvalue $\lambda$ replaced by $\max\{\lambda, m_t\}$, and $m_t$ being the smallest value for $||\tilde{H}_{t}^{-1}\nabla f(\vct z_t)||_2\leq\frac{M}{\rho}$. There exists an explicit function $T(\rho, M, R)\leq 5R^2\rho^2/M^2+1 
$ such that 
$\nabla f(\vct z_t)\leq \frac{M^2}{2\rho}$ holds for any $f$ and any $t\geq T(\rho, M, R)$.
\end{proposition}
\begin{proof}
For convenience, let $\tilde{\vct r}_t\triangleq -(\nabla^2 f(\vct z_t))^{-1}\nabla f(\vct z_t)$ and ${\vct r}_t\triangleq -\tilde{H}_{t}^{-1}\nabla f(\vct z_t)$. 
To investigate the evolution of gradients, 
we integrate the Lipschitz Hessian condition and obtain that 
\begin{align}\label{ineq:basb1}
||\nabla f(\vct z_{t+1})-\nabla f(\vct z_t)-\nabla^2 f(\vct z_t) \, {\vct r_t} ||_2\leq \frac{1}{2}\rho ||{\vct r}_t||_2^2.  
\end{align}
From the definition of ${\vct r_t}$, if $||\tilde{\vct r}_t||_2\leq M/\rho$, we have 
\begin{align}\label{ineq:appb_earl}
\left|\left|\nabla f(\vct z_{t+1})\right|\right|_2\leq \frac{1}{2}\rho||\tilde{\vct r}_t||_2^2\leq \frac{M^2}{2\rho}. 
\end{align}
Note that by strong convexity, when the above bound holds, 
\begin{align}\label{ineq:appb_earl2}
||\tilde{\vct r}_{t+1}||_2\leq\frac{ \left|\left|\nabla f(\vct z_{t+1})\right|\right|_2}{M}\leq \frac{M}{2\rho}\leq \frac{M}{\rho}.  
\end{align}
Hence, once $||\tilde{\vct r}_t||_2$ reaches below $ M/\rho$ for some $t_0$, our desired bound on $\left|\left|\nabla f(\vct z_{t+1})\right|\right|_2$ remain hold for any $t>t_0$. Therefore, for the purpose of our proof, we can focus on the case where $||\tilde{\vct r}_t||_2> M/\rho$ and show that this condition can only hold for boundedly many iterations.

Consider any fixed function $f\in\mathcal{F}(\rho,M,R)$, let $\vct x^*$ denote its global minimum point. A crucial step in our proof is to show that 
\begin{align}\label{ineq:ppbs_cri}
    (\vct x^*-\vct{z}_{t}) \cdot  \vct r_t\geq 
    0.6||\vct r_t||_2^2.
\end{align}
For brevity, 
let $\beta$ denote the minimum eigenvalue of $\tilde{H}_t$, and $\tilde{f}$ denote the following quadratic approximation.
\begin{align*}
    \tilde{f}(\vct x)\triangleq f(\vct z_t) + (\vct x- \vct z_t) \nabla f(\vct z_t) +\frac{1}{2}(\vct x- \vct z_t)(\nabla^2 f(\vct z_t)) (\vct x- \vct z_t). 
\end{align*}
We apply the strong convexity condition at point 
$\vct y\triangleq 
\vct z_{t+1}-0.4 \beta \tilde{H}_t^{-1} {\vct r}_t$. Recall that $f$ is minimized at $\vct x^*$, we have  
\begin{align}\label{ineq:ppzeron_scb}
0\geq f(\vct x^*)-f(\vct y) \geq &\, \nabla f(\vct y) \cdot (\vct x^*-\vct{y})+\frac{M}{2} ||\vct x^*-\vct{y}||_2^2. 
\end{align}
By integrating the Lipschitz Hessian, similar to inequality \eqref{ineq:basb1}, $ \nabla f(\vct y)$ can be approximated with $ \nabla f(\vct z_t) + \nabla^2 f(\vct z_t) (\vct y-\vct z_t)=\nabla \tilde{f}(\vct y)$. 
Formally, 
\begin{align}\label{ineq:ppzero_intlshort}
  || \nabla f(\vct y)  -\nabla \tilde{f}(\vct y)||_2\leq \frac{1}{2}\rho ||\vct y-\vct z_t||_2^2. 
\end{align}
Hence, inequality \eqref{ineq:ppzeron_scb} implies that 
\begin{align}\label{ineq:ppzero_rhs461}
 0\geq &\, \nabla \tilde{f}(\vct y) \cdot (\vct x^*-\vct{y})+\frac{M}{2} ||\vct x^*-\vct{y}||_2^2-\frac{1}{2}\rho ||\vct y-\vct z_t||_2^2||\vct x^*-\vct{y}||_2.
\end{align}
To characterize the terms in the above inequality, we first note that
\begin{align*}
||\vct y-\vct z_t||_2^2&=||\vct r_t||_2^2 -0.8 \vct r_t \cdot \tilde{H}_t^{-1}\beta \vct r_t + 0.16||\tilde{H}_t^{-1}\beta \vct r_t||_2^2 \nonumber\\&\leq||\vct r_t||_2^2- 0.64||\tilde{H}_t^{-1}\beta \vct r_t||_2^2.
\end{align*}
For convenience, we denote $c\triangleq ||\tilde{H}_t^{-1}\beta \vct r_t||_2 /||\vct r_t||_2$. We have that $c\in(0,1]$ and 
\begin{align}\label{ineq:ppzero_bound472}
||\vct y-\vct z_t||_2^2&\leq \left(1-0.64c^2\right)||\vct r_t||_2^2.
\end{align}
We also consider the following vector,
\begin{align*}
    \vct q&\triangleq (\nabla^2 f(\vct z_t))^{-1} \left( \nabla \tilde{f}(\vct y)+(\beta +0.36 cM)\vct r_t - 0.6cM(\vct y-\vct z_t)\right)\\&= 0.6 \tilde{H}_t^{-1} \left(\beta \vct r_t +0.4c M \left(\tilde{H}_t^{-1}\beta \vct r_t -\vct r_t \right)\right), 
\end{align*}
of which the L2 norm is no greater than $0.6 c||\vct r_t||_2\leq {0.6cM}/{\rho }$, which can be proved in the eigenbasis of $\nabla^2 f(\vct z_t)$.
By Cauchy's inequality, we have that 
\begin{align}
  \vct q \cdot \nabla\tilde{f}(\vct x^*)
& \geq -||\vct q||_2 ||\nabla\tilde{f}(\vct x^*)||_2  \nonumber\\
&\geq -\frac{0.6cM}{\rho } ||\nabla\tilde{f}(\vct x^*)||_2. \label{ineq:ppzero_rhs46}
\end{align}
Note that 
$\vct x^*-\vct{y}=(\nabla^2 f(\vct z_t))^{-1} \left(\nabla\tilde{f}(\vct x^*)-\nabla\tilde{f}(\vct y)\right)$. The LHS of the above inequality can be written as $ (\vct x^*-\vct{y})\cdot (\nabla^2 f(\vct z_t)) \cdot \vct q + \vct q \cdot \nabla\tilde{f}(\vct y)$, where the first term contains $\nabla \tilde{f}(\vct y) \cdot (\vct x^*-\vct{y})$, and the second term is bounded as follows.
\begin{align}\label{ineq:ppzero_bd47n}
    \vct q \cdot \nabla\tilde{f}(\vct y)=&\,-0.6 \tilde{H}_t^{-1} \left( \tilde{H}_t^{-1}\beta^2\vct r_t+(\beta-0.4cM) (\vct r_t-\tilde{H}_t^{-1}\beta \vct r_t)\right)\nonumber\\& \, \cdot \left(0.4 \beta \vct r_t+0.6\left(\tilde{H}_t-\nabla^2 f(\vct z_t)\right)\vct r_t\right)\nonumber\\\leq &\,-0.6 \tilde{H}_t^{-1} \cdot  \tilde{H}_t^{-1}\beta^2\vct r_t \cdot 0.4 \beta \vct r_t\nonumber \\ =&\, -{0.24c^2\beta ||\vct r_t||_2^2}.
\end{align}
On the other hand, 
observe that inequality \eqref{ineq:ppzero_intlshort} holds for any generic $\vct y\in\mathbb{R}^d$. The RHS of inequality \eqref{ineq:ppzero_rhs46} can be characterized as follows.
\begin{align}
  || \nabla \tilde{f}(\vct x^*)||_2&=|| \nabla f(\vct x^*)-\nabla \tilde{f}(\vct x^*)||_2\nonumber\\&\leq \frac{1}{2}\rho ||\vct x^*-\vct z_t||_2^2 \nonumber\\
  &= \frac{1}{2}\rho ||\vct x^*-\vct y||_2^2  +\rho (\vct x^*-\vct y)\cdot (\vct y-\vct z_t) +\frac{1}{2}\rho ||\vct y-\vct z_t||_2^2. \label{ineq:ppzero_rhs47}
\end{align}
Therefore, by combining 
inequalities \eqref{ineq:ppzero_rhs461}, \eqref{ineq:ppzero_rhs46}, and 
\eqref{ineq:ppzero_rhs47},  we have that
\begin{align*}
     (\vct x^*-\vct{y})\cdot \left( \beta +0.36c M  \right) \vct r_t \geq  &\, - \vct q \cdot \nabla\tilde{f}(\vct y)+ (0.5-0.3c)M||\vct x^*-\vct{y}||_2^2\\ &-0.5\rho ||\vct y-\vct z_t||_2^2\cdot ||\vct x^*-\vct{y}||_2-{0.3cM}||\vct y-\vct z_t||_2^2\\
  \geq &\, - \vct q \cdot \nabla\tilde{f}(\vct y)- \frac{\rho^2 ||\vct y-\vct z_t||_2^4}{(8-2.4c)M} -{0.3cM}||\vct y-\vct z_t||_2^2,  
\end{align*}
where the second line is obtained by taking the infimum w.r.t. $||\vct x^*-\vct{y}||_2$.  
Then, 
we apply inequalities \eqref{ineq:ppzero_bound472},   \eqref{ineq:ppzero_bd47n}, and $||\vct r_t||_2\leq M/\rho$  
to obtain the following bound.  
\begin{align*}
     (\vct x^*-\vct{y})\cdot  \vct r_t 
  \geq &\, \frac{{0.24c^2}\beta - \frac{M \left(1-0.64c^2\right)^2}{(8-2.4c)} -{0.3cM}\left(1-0.64c^2\right)}{\beta +0.36c M  }\cdot ||\vct r_t||_2^2.  
\end{align*}
Note that the above bound is non-decreasing w.r.t. $\beta$, and our construction implies $\beta\geq M$. We can substitute $\beta$ in the above inequality with $M$. Further, note that 
\begin{align*} 
 (\vct{y}-\vct z_t)\cdot \vct r_t = &\, ||\vct r_t||_2^2- \vct r_t\cdot 0.4\beta \tilde{H}_t^{-1} \vct r_t\\ \geq  &\,   \left(1-0.4 {c}\right)||\vct r_t||_2. 
\end{align*}
We have obtained a lower bound of $(\vct{x}^*-\vct z_t)\cdot \vct r_t$ as a function of $c$. This dependency is removed by taking the infimum, i.e., 
\begin{align*}
     (\vct x^*-\vct z_t)\cdot  \vct r_t 
  \geq &\, \inf_{c\in(0,1]}\left(\frac{{0.24c^2} - \frac{ \left(1-0.64c^2\right)^2}{(8-2.4c)} -{0.3c}\left(1-0.64c^2\right)}{1 +0.36c }+1-0.4 {c}\right)\cdot ||\vct r_t||_2,  
\end{align*}
then inequality \eqref{ineq:ppbs_cri} is obtained. 




We use this key inequality to obtain the following recursion rule. 
\begin{align}
||\vct x^*-\vct{z}_{t}||_2^2-||\vct x^*-\vct{z}_{t+1}||_2^2 =& \,   2\vct r_t \cdot (\vct x^*-\vct{z}_{t}) - ||\vct r_t||_2^2 \geq 
0.2||\vct r_{t}||_2^2. \nonumber 
\end{align}
Recall that for $f\in\mathcal{F}(\rho,M,R)$, we assumed that $||\vct x^*||_2\leq R$. Therefore, for $\vct z_1=\vct 0$, the above recursion implies that the inequality $||\tilde{\vct r}_{t}||_2>M/\rho$ can hold for no greater than $5R^2\rho^2/M^2$ iterations as $||\vct x^*-\vct{z}_{t}||_2^2$ has to be non-negative for any $t$. Hence, based on the earlier discussion, we have proved that either $||\tilde{\vct r}_t||_2\leq  M/\rho$ or $\nabla f(\vct z_{t+1})=\vct 0$  for all $t\geq  5R^2\rho^2/M^2$ and all $f\in\mathcal{F}(\rho,M,R)$.
Recall inequality \eqref{ineq:appb_earl}, this implies that 
$||\nabla f(\vct z_t)||_2\leq \frac{M^2}{2\rho}$ for all $t\geq 5R^2\rho^2/M^2+1$. 
\end{proof}
\begin{remark}
Recall the recursion provided by inequality \eqref{ineq:appb_earl} and \eqref{ineq:appb_earl2}. The gradients for the sequence $\vct z_t$ decay double-exponentially once they are sufficiently close to zero. Hence, Proposition \ref{prop:app_zeroerr} proves that it takes finitely many iterations for 
the bootstrapping stage of Algorithm \ref{alg:1gd3_opt} to get arbitrarily close to $\vct x^*$ in the zero-error case. 
\end{remark}
\subsection{Generalization to the noisy case}
Now we prove that, given bounded many iterations, the bootstrapping stage in Algorithm \ref{alg:1gd3_opt} provides an $\vct x_{\textup{B}}$ that is sufficiently close to $\vct x^*$ with high probability even in the presence of noise. For convenience, let $\vct x_k^{(\textup{B})}$ denote the realization of vector $\vct x$ at the end of the $k$th iteration in the bootstrapping stage and $N\triangleq \lfloor T^{0.1}\rfloor$ denote the number of iterations. Therefore, we have $\vct x_{\textup{B}}=\vct x_N^{(\textup{B})}$.
Further, we define $\vct x_0^{(\textup{B})}\triangleq \vct 0$. We let $\vct{m}_k, H_{k}$ denote the realization of $\hat{\vct m}, H_{m^*}$ in the $(k+1)$th iteration of the bootstrapping stage. 
Therefore, we have $\vct x_{k+1}^{(\textup{B})}=\vct x_{k}^{(\textup{B})}-H_k^{-1}\vct m_k$.
As a Benchmark for our analysis, we use  $\tilde{H}_k$ to denote the value of  $H_{m^*}$ in the zero-error case, i.e., they denotes the value of $H_{m^*}$ under the special case of $\hat{\vct m}=\nabla f\left(\vct x_{k}^{(\textup{B})}\right)$ and $\hat{H}=\nabla^2 f\left(\vct x_{k}^{(\textup{B})}\right)$. Hence, the update in the zero-error case can be denoted as $\vct{r}_k\triangleq -\tilde{H}_k^{-1} \nabla f\left(\vct x_{k}^{(\textup{B})}\right)$.

We let $E_k$ be the indicator function of the event where there exists an $j< k$ such that 
 $\big|\big|\vct x_{j+1}^{(\textup{B})}-\vct x_{j}^{(\textup{B})}-\vct{r}_{j}\big|\big|_2\geq {M T^{-0.2}}/{\rho}$. Intuitively, $E_k=0$ describes the event that the optimization steps can be characterized similar to the zero-error case.
Notice that $E_k$ is non-decreasing. We have either $E_{N}=0$, or $E_{k_0}=1$ for some $k_0\in\{1,2,...,N\}$.  
We provide the analysis separately for these two cases.

For the first case, i.e., when $E_{N}=0$, we can follow the earlier arguments and prove the following proposition (see Appendix \ref{app:ppc2e} for details). 
\begin{proposition}\label{prop:ac2e}
For any function $f\in\mathcal{F}(\rho,M,R)$ and any sequence $\vct x_{0}^{(\textup{B})},\vct x_{1}^{(\textup{B})},..., \vct x_{N-1}^{(\textup{B})}\in\mathbb{R}^d$ that satisfies $\vct x_{0}^{(\textup{B})}=\vct 0$ and $E_{N-1}=0$, we have 
$||\vct r_{N-1}||_2\leq {2M T^{-0.2}}/{\rho}$ when $T$ is sufficiently large.     
\end{proposition}
Recall the definition of $E_N=0$. The above proposition immediately implies that 
\begin{align*}
\left|\left|\vct x_{N}^{(\textup{B})}-\vct x_{N-1}^{(\textup{B})}\right|\right|_2\leq {3M T^{-0.2}}/{\rho}<M/\rho
\end{align*} when $T$ is large. Hence, in such cases, $\vct x_{N}^{(\textup{B})}$ is obtained by the  Newton update. Formally, if $\hat{H}$ denotes the estimator returned by the HessianEst function in the $N$th iteration, we have 
\begin{align}\label{ineq:ppcf3pre2}
     (\vct x_{N}^{(\textup{B})}-\vct x_{N-1}^{(\textup{B})})\cdot \hat{H}=-\vct m_{N-1},
\end{align} 
Therefore, by applying the above results to the Lipschitz Hessian condition, we have
\begin{align}\label{ineq:ppcf3}
 \left|\left|\nabla f(\vct x_{N}^{(\textup{B})})\right|\right|_2&\leq  \left|\left|\nabla f(\vct x_{N-1}^{(\textup{B})})+(\vct x_{N}^{(\textup{B})}-\vct x_{N-1}^{(\textup{B})})\cdot \nabla^2 f(\vct x_{N-1}^{(\textup{B})})\right|\right|_2+\frac{\rho}{2}||\vct x_{N}^{(\textup{B})}-\vct x_{N-1}^{(\textup{B})}||_2^2\nonumber\\
 &\leq  
\left|\left|\nabla f(\vct x_{N-1}^{(\textup{B})})-\vct m_{N-1}+ (\vct x_{N}^{(\textup{B})}-\vct x_{N-1}^{(\textup{B})})\cdot \left(\nabla^2 f(\vct x_{N-1}^{(\textup{B})})-\hat{H}\right)\right|\right|_\textup{F} 
+\frac{9M^2}{2\rho T^{0.4}}\nonumber\\
 &\leq  
\left|\left|\nabla f(\vct x_{N-1}^{(\textup{B})})-\vct m_{N-1}\right|\right|_2 + \frac{3M}{\rho T^{0.2}}\cdot \left|\left|\nabla^2 f(\vct x_{N-1}^{(\textup{B})})-\hat{H}\right|\right|_\textup{F} 
+\frac{9M^2}{2\rho T^{0.4}}. 
\end{align} 
Hence, by direct integration of the tail bounds in Theorem \ref{thm:hest}, we can conclude that 
\begin{align}\label{appine_pppart1}
&\ \ \ \ \ \limsup_{T\rightarrow\infty }\mathbb{E}\left[\left|\left|\nabla f(\vct x_{N}^{(\textup{B})})\right|\right|_2^3 \cdot \mathbbm{1}(E_N=0)\cdot T^{\frac{2}{3}}\right]=0. 
\end{align}

Now we consider the second case, i.e., when $E_N=1$. By its definition, we must have the event of $E_k=0$ to $E_{k+1}=1$ for a unique $k\in\{0,1,...,N-1\}$, which implies that $\big|\big|\vct x_{k+1}^{(\textup{B})}-\vct x_k^{(\textup{B})}-\vct{r}_k\big|\big|_2\geq {M T^{-0.2}}/{\rho}$. We prove that conditioned on any of these events, the random variable $||\nabla f(\vct x_{k+1}^{(\textup{B})})||_2$ has a super-polynomial tail, which contributes vanishingly to their moments in the asymptotic sense. 
Formally, let
\begin{align*}
    M_k&\triangleq \mathbb{E}\left[ ||\nabla f(\vct x_{k+1}^{(\textup{B})})||_2^3\cdot\mathbbm{1}(E_{k+1}=1, E_{k}=0 )\right],
\end{align*}
We aim to prove that 
\begin{align}\label{ineq:ppc_key_summ50}
   \limsup_{T\rightarrow \infty}  \max_{k\in\{0,1,...,N-1\}} M_k \cdot N   T^{\frac{2}{3}}=0. 
\end{align}
Consider any fixed $k\in\{0,1,...,N-1\}$ and conditioned on any realization of $\vct x_k^{(\textup{B})}$, we 
characterize the distribution of $\vct x_{k+1}^{(\textup{B})}$ 
by providing the following proposition, which is proved in Appendix \ref{app:pp13br}. 
\begin{proposition}\label{prop:13_br}
    Consider any vectors $\vct m, \vct m'\in\mathbb{R}^n$, any positive definite matrices $H, H' \in\mathbb{R}^n$ 
    with all eigenvalues lower bounded by $M$, and any fixed parameter $R_0\in\mathbb{N}_+$. Let $H_{m^*}$ be the symmetric matrix sharing the same eigenbasis of $H$ but with each eigenvalue $\lambda$ replaced with $\max\{\lambda,m^*\}$, where $m^*$ is chosen to be the smallest value such that $||H_{m^*}^{-1}\vct m||_2\leq R_0$. Let $H_{m'^*}'$ be defined correspondingly for $\vct m'$ and $H'$. We have that
    \begin{align}\label{ineq:pp13_main}
  \left|\left|H_{m^*}^{-1}\vct m-H_{m'^*}'^{-1}\vct m'
    \right|\right|_2^2\leq\frac{2R_0}{M}\cdot\left( { \left|\left|\vct m-\vct m'\right|\right|_2}+{R_0}\cdot{ \left|\left|H-H'\right|\right|_{\textup{F}}}\right).
    \end{align}
    Furthermore,   when $\left|\left|H_{m^*}^{-1}\vct m-H_{m'^*}'^{-1}\vct m'
    \right|\right|_2>0$, we have
        \begin{align}\label{ineq:pp13_2}
    &\left|\left|H'_{m'^*}
    \left(H_{m^*}^{-1}\vct m-H_{m'^*}'^{-1}\vct m'\right)\right|\right|_2\nonumber\\ &\ \ \ \ \ \ \ \ \ \ \ \ \ \ \ \ \ \  \ \ \ \ \ \ \ \ \ \ \ \leq \left(3+\frac{2R_0}{||H_{m^*}^{-1}\vct m-H_{m'^*}'^{-1}\vct m'||_2}\right) \left(  \left|\left|\vct m-\vct m'\right|\right|_2+R_0 \left|\left|H-H'\right|\right|_{\textup{F}}\right).
    \end{align}
\end{proposition}
By choosing $R_0=M/\rho$, $H_{m^*}=H_{N-1}$, $H'_{m'^*}=\nabla^2 f(\vct x_{N-1}^{(\textup{B})})$, $\vct m=\vct m_{N-1}$,  and $\vct m'=\nabla f(\vct x_{N-1}^{(\textup{B})})$ for Proposition \ref{prop:13_br}, the 
condition of $E_{k+1}$ can be characterized by the estimation errors of the gradient and Hessian. 
For brevity, we define
\begin{align*}
    \Psi\triangleq \left|\left|\vct m-\vct m'\right|\right|_2+R_0 \left|\left|H-H'\right|\right|_{\textup{F}}.
\end{align*}
We also let $\beta$ denote the minimum eigenvalue of ${H}_k$. The condition of $E_{k+1}=1$ and $E_{k}=0$ implies that $||H_{m^*}^{-1}\vct m-H_{m'^*}'^{-1}\vct m'||_2\geq R_0T^{-0.2}$, which implies that $\Psi\geq \frac{\beta R_0 T^{-0.2}}{3+2 T^{0.2}}$ according to inequality \eqref{ineq:pp13_2}. Hence, $M_k$ can be bounded as follows.
\begin{align*}
    M_k&\leq \mathbb{E}\left[ ||\nabla f(\vct x_{k+1}^{(\textup{B})})||_2^3\cdot\mathbbm{1}\left(\Psi\geq \frac{\beta R_0 T^{-0.2}}{3+2 T^{0.2}}\right)\right],
\end{align*}


On the other hand, by generalizing inequality \eqref{ineq:ppcf3}, we have
\begin{align}\label{ineq:ppcf3_gen}
 \left|\left|\nabla f(\vct x_{k+1}^{(\textup{B})})\right|\right|_2&\leq \left|\left|\nabla f(\vct x_{k}^{(\textup{B})})+(\vct x_{k+1}^{(\textup{B})}-\vct x_{k}^{(\textup{B})})\cdot \nabla^2 f(\vct x_{k}^{(\textup{B})})\right|\right|_2+\frac{\rho}{2}||\vct x_{k+1}^{(\textup{B})}-\vct x_{k}^{(\textup{B})}||_2^2\nonumber\\
  &\leq  \left|\left|\nabla f(\vct x_{k}^{(\textup{B})})-\vct m_{k}\right|\right|_2+\left|\left|\vct x_{k+1}^{(\textup{B})}-\vct x_{k}^{(\textup{B})}\right|\right|_2\cdot \left|\left|\nabla^2 f(\vct x_{k}^{(\textup{B})})- 
 \hat{H}\right|\right|_\textup{F}\nonumber\\&\ \ \ \ +\left|\left|\left(\vct x_{k+1}^{(\textup{B})}-\vct x_{k}^{(\textup{B})}\right)\left(\hat{H}- 
 H_k\right)\right|\right|_2+\frac{\rho}{2}||\vct x_{k+1}^{(\textup{B})}-\vct x_{k}^{(\textup{B})}||_2^2\nonumber\\&\leq  
\Psi +R_0\beta+\frac{R_0M}{2}. 
\end{align}
Therefore,
\begin{align}
   \limsup_{T\rightarrow \infty}  M_k \cdot N  T^{\frac{2}{3}}&\leq \limsup_{T\rightarrow \infty} \mathbb{E}\Bigg[\left(\Psi+R_0\beta+ \frac{R_0 M}{2}\right)^3\cdot\mathbbm{1}\left(\Psi\geq \frac{\beta R_0 T^{-0.2}}{3+2 T^{0.2}} \right)\cdot  N  T^{\frac{2}{3}}
 \Bigg]\nonumber\\
 &=0. 
\end{align}
Since the above bounds are uniform over the index $k$, equation \eqref{ineq:ppc_key_summ50} is implied. 
The above arguments also show that 
\begin{align}
   \limsup_{T\rightarrow \infty}  \mathbb{P}[E_{N}=1]\cdot  N^3 T^{\frac{2}{3}}&\leq \limsup_{T\rightarrow \infty} N\cdot \max_k \mathbb{P}[E_{k+1}=1, E_k=0] \cdot  N^3 T^{\frac{2}{3}} 
 =0. 
\end{align}

~

So far, we have proved that the moments of the gradient norm $\left|\left|\nabla f(\vct x_{k}^{(\textup{B})})\right|\right|_2$ is bounded after entering the $E_k=1$ phase. We proceed to bound their contribution to the $N$th iteration. 
 To that end, we denote
\begin{align*}
    G_k&\triangleq \mathbb{E}\left[ ||\nabla f(\vct x_{k}^{(\textup{B})})||_2^3\cdot\mathbbm{1}(E_{k}=1)\right].
\end{align*}
This sequence is initialized with $G_0=0$ by definition. 
We establish the following recursion for sufficiently large $T$. 
\begin{align*}
    G_{k+1}&\leq  G_k\left(1+\frac{1}{N}\right)+6N^2(\rho R_0^2)^3\cdot \mathbb{P}[E_{N}=1]+M_k.
\end{align*}

We note that conditioned on any fixed $\vct  x_k^{(\textup{B})}$ the gradient norm function $||\nabla f(\vct x_{k+1}^{(\textup{B})})||_2$ can be approximated with its linear expansion. Formally, let $\tilde{g}(\vct x)\triangleq \nabla f(\vct x_{k}^{(\textup{B})})+ (\vct x-\vct x_{k}^{(\textup{B})})\cdot \nabla^2 f(\vct x_{k}^{(\textup{B})}),$ we have 
\begin{align*}
   \left|\left|\nabla f(\vct x_{k+1}^{(\textup{B})})\right|\right|_2&\leq \left|\left|\tilde{g}(\vct x_{k+1}^{(\textup{B})})\right|\right|_2+\frac{1}{2}\rho \left|\left|\vct x_{k+1}^{(\textup{B})}-\vct x_{k}^{(\textup{B})}\right|\right|_2^2\\&\leq  \left|\left|\tilde{g}(\vct x_{k+1}^{(\textup{B})})\right|\right|_2+\frac{1}{2}\rho R_0^2.
\end{align*}
Then, in the eigenbasis of $\hat{H}$, it is clear that 
\begin{align*}
\left|\left|\tilde{g}(\vct x_{k+1}^{(\textup{B})})\right|\right|_2\leq &\,\left|\left|\nabla f(\vct x_{k}^{(\textup{B})})+ (\vct x_{k+1}^{(\textup{B})}-\vct x_{k}^{(\textup{B})})\cdot \hat{H}\right|\right|_2\\&+\left|\left| (\vct x_{k+1}^{(\textup{B})}-\vct x_{k}^{(\textup{B})})\cdot \left(\hat{H}-\nabla^2 f(\vct x_{k}^{(\textup{B})})\right)\right|\right|_2\\
\leq &\, \left|\left|\nabla f(\vct x_{k}^{(\textup{B})})\right|\right|_2+\left|\left|\vct m_k-\nabla f(\vct x_{k}^{(\textup{B})})\right|\right|_2+R_0\left|\left|\hat{H}-\nabla^2 f(\vct x_{k}^{(\textup{B})})\right|\right|_{\textup{F}}.
\end{align*}
Recall that by Theorem \ref{thm:hest}, when $T$ is sufficiently large, the moments of $\left|\left|\vct m_k-\nabla f(\vct x_{k}^{(\textup{B})})\right|\right|_2+R_0\left|\left|\hat{H}-\nabla^2 f(\vct x_{k}^{(\textup{B})})\right|\right|_{\textup{F}}$ is upper bounded by any fixed quantity. 
Therefore, as a rough estimate, we have
\begin{align*}
\mathbb{E}\left[ ||\nabla f(\vct x_{k+1}^{(\textup{B})})||_2^3\cdot\mathbbm{1}(E_{k}=1)\right]&\leq \mathbb{E}\left[ ||\nabla f(\vct x_{k}^{(\textup{B})})+\rho R_0^2||_2^3\cdot\mathbbm{1}(E_{k}=1)\right]\\
&\leq G_k\left(1+\frac{1}{N}\right)+6N^2(\rho R_0^2)^3 \cdot \mathbb{P}[E_{k}=1]
\end{align*}
when  $T$ is sufficiently large. Consequently, our needed recursion is implied by the monotonicity of $E_k$, and we have 
\begin{align}
   \limsup_{T\rightarrow \infty}  G_N \cdot T^{\frac{2}{3}}&\leq \limsup_{T\rightarrow \infty} \left( \max_{k} M_k +6N^2(\rho R_0^2)^3 \cdot \mathbb{P}[E_{k}=1]\right)\cdot  N \left(1+\frac{1}{N}\right)^N T^{\frac{2}{3}}\nonumber\\
 &=0. 
\end{align}
Then, inequality \eqref{eq:fststbd} is obtained by combining strong convexity, the above bound, and inequality
\eqref{appine_pppart1}. 
\begin{remark}
Note that 
compared to the simple regret guarantee stated in inequality \eqref{eq:fststbd}, 
we have essentially proved a stronger statement that the moments of the gradient at the outcome of the bootstrapping stage follow similar power decay laws. 
Therefore, while we presented a final stage algorithm that uses non-isotropic sampling to be compatible with general bootstrapping stages, our specific bootstrapping stage actually allows for the use of isotropic (hyperspherical) sampling for gradient estimation in the final stage.
\end{remark}

\section{Proofs of some useful propositions}\label{app:ppa_ele}
\subsection{Proof of Proposition \ref{prop:pa_ele}}
\begin{proof}
Recall that all $ z_j$'s have zero expectations. 
By subgaussianity, we have that all even moments of $z_j$ are 
bounded as follows.
\begin{align}\mathbb{E}\left[ z_j^{2\ell}\right]&= \int_{K=0}^{{+\infty}}2\ell K^{2\ell-1} \mathbb{P}\left[\left|z_j\right|\geq K\right] \textup{d}K\nonumber\\&\leq \int_{K=0}^{{+\infty}} 2\ell K^{2\ell-1} 
\min\left\{2\exp\left({-\frac{K^2}{\sigma_j^2}}\right),1 \right\}
\textup{d}K \nonumber\\
&\leq  \begin{cases}
      (1+\ln 2)\, \sigma_j^{2} & \text{if $\ell=1$,}\\
       (2+2\ln 2+\ln^2 2)\, \sigma_j^{4} \ \  \ \ & \text{if $\ell=2$,}\\
      2 \cdot \ell! \,  \sigma_j^{2\ell}  & \text{if $\ell>2$.}
    \end{cases}\label{ineq:appd_boundind}
\end{align}
Using AM-GM inequality, the odd moments of $z_j$ can then be bounded using the even moments. 
Specifically,   
\begin{align*}
\mathbb{E}\left[ z_j^{2\ell+1}\right]\leq \frac{1}{2s}\mathbb{E}\left[ z_j^{2\ell}\right]+\frac{s}{2}\,\mathbb{E}\left[ z_j^{2\ell+2}\right].
\end{align*}
Therefore, we have obtained the following upper bounds for the moment-generating function.  
\begin{align}
\mathbb{E}\left[\exp( sz_j)\right]&= 1+ \sum_{m=2}^{\infty}\frac{s^{m}}{m!}\mathbb{E}\left[ z_j^{m}\right]\nonumber\\
&\leq 1+ \frac{7s^{2}}{12}\mathbb{E}\left[ z_j^{2}\right] + \sum_{\ell=2}^{\infty}\,\left(2\ell+2+\frac{1}{2\ell+1}\right)\frac{s^{2\ell}}{(2\ell)!\cdot 2}\,\mathbb{E}\left[ z_j^{2\ell}\right].\nonumber
\end{align}
Applying inequality \eqref{ineq:appd_boundind},  the  expression above can be bounded with a series of $(s\sigma_j)^2$. The coefficient of each $(s\sigma_j)^{2\ell}$ is no greater than $\frac{1}{\ell!}$, which can be verified numerically for $\ell\leq 2$ and inductively for $\ell\geq 3$. Hence, we have 
\begin{align}
\mathbb{E}\left[\exp( sz_j )\right]&\leq  \sum_{\ell=0}^{\infty}\,\frac{(s\sigma_j)^{2\ell}}{\ell!}=e^{(s\sigma_j)^2}.
\end{align}
Because $z_j$'s are independent, 
\begin{align*}\mathbb{E}\left[\exp\left( s \sum_j z_j\right)\right]=\prod_j \mathbb{E}\left[\exp( sz_j)\right]\leq  \exp\left(s^2\sum_j \sigma_j^2\right).
\end{align*}
Inequality \eqref{eq:appd_res} is implied by Markov's bound.  
Specifically, for any $K\geq 0$, 
\begin{align*}
\mathbb{P}\left[\sum_{j=1}^{k} z_j\geq K\right]&\leq \inf_{s\geq 0} {\mathbb{E}\left[\exp\left( s \sum_j z_j\right)\right]}\cdot {\exp\left(-sK\right)}\\
&\leq \inf_{s\geq 0}  {\exp\left(s^2\sum_j \sigma_j^2-sK\right)}\\
&=\exp\left( -\frac{K^2}{4\sum_j \sigma_j^2}\right).
\end{align*}
For the same reason, we also have 
\begin{align*}
\mathbb{P}\left[\sum_{j=1}^{k} z_j\leq -K\right]&\leq \exp\left( -\frac{K^2}{4\sum_j \sigma_j^2}\right).
\end{align*}
Hence, by union bound, 
\begin{align}
\mathbb{P}\left[\left|\sum_{j=1}^{k} z_j\right|\geq K\right]&\leq \mathbb{P}\left[\sum_{j=1}^{k} z_j\geq K\right]+\mathbb{P}\left[\sum_{j=1}^{k} z_j\leq -K\right]
\leq 2\exp\left( -\frac{K^2}{4\sum_j \sigma_j^2}\right).
\end{align}
\end{proof}

\subsection{Proof of Proposition \ref{prop:pa_ele2}}
\begin{proof}
By subexponentiality, the moment-generating function of each $|z_j|$ is bounded as follows for any $s<\frac{1}{\sigma_j}$.
\begin{align}
\mathbb{E}[\exp(s|z_j|)]&=1+\int_{K=0}^{{+\infty}}s \exp({sK}) \cdot \mathbb{P}\left[\left|z_j\right|\geq K\right] \textup{d}K\nonumber\\&\leq 1+\int_{K=0}^{{+\infty}} s \exp({sK}) \cdot  
\min\left\{2\exp\left({-\frac{K}{\sigma_j}}\right),1 \right\}dK\nonumber\\
&=\frac{2^{s\sigma_j}}{1-s{\sigma_j}}. 
\end{align}
Because $z_j$'s are independent,
\begin{align}
\mathbb{E}\left[\exp\left(s\left|\sum_j z_j\right|\right)\right]&\leq \mathbb{E}\left[\exp\left(s\sum_j \left|z_j\right|\right)\right] 
=\prod_j \mathbb{E}[\exp(s| z_j|)]\nonumber\\&\leq \frac{2^{s\sum_j \sigma_j}}{\prod_j(1-s\sigma_j)}. 
\end{align}
We choose $s=1/(3\sum_{j}\sigma_j)$, note that $s\sigma_j\leq1/3$, we have $(1-s\sigma_j)\geq \left(\frac{2}{3}\right)^{3s\sigma_j}$. Hence,
\begin{align}
\mathbb{E}\left[\exp\left(s\left|\sum_j z_j\right|\right)\right]&\leq e^{(\ln2-3\ln \frac{2}{3} )(s\sum_j \sigma_j )}= 3/2^{\frac{2}{3}}<2.\nonumber
\end{align}
Then, inequality \eqref{eq:appd_res2} is implied by Markov's bound, i.e.,   
\begin{align*}
\mathbb{P}\left[\left|\sum_{j=1}^{k} z_j\right|\geq K\right]&\leq \mathbb{E}\left[\exp\left(s\left|\sum_j z_j\right|\right)\right]\cdot {\exp\left(-sK\right)}\\
&\leq 2\exp\left( -\frac{K}{3\sum_j \sigma_j}\right).
\end{align*}
\end{proof}

\subsection{Proof of Proposition \ref{prop:ac2e}}\label{app:ppc2e}
To prove the proposition for sufficiently large $T$, we focus on the regime where $N\geq 10R^2\rho^2/M^2+2$. We first use proof by contradiction to show the existence of $k_0\leq 10R^2\rho^2/M^2$ such that  $||\vct r_{k_0}||_2< M/\rho$. Assume the contrary, we have 
$\left|\left|\vct{r}_k\right|\right|_2\geq M/\rho$ 
for all $k\leq 10R^2\rho^2/M^2$. 
Recall we have proved earlier that (see inequality \eqref{ineq:ppbs_cri})
\begin{align}\label{ineq:ppbs_cric}
    \left(\vct x^*-\vct x_k^{(\textup{B})}\right) \cdot  \vct{r}_k \geq 
    0.6||\vct r_k||_2^2.
\end{align}
This assumption implies that $R\geq 0.6{M}/{\rho} $ and $\left|\left|\vct x^*-\vct x_k^{(\textup{B})}\right|\right|_2\geq 0.6{M}/{\rho}$ for all $k\leq 10R^2\rho^2/M^2$.

We characterize the evolution of $x_k^{(\textup{B})}$. 
By 
Cauchy's inequality and inequality \eqref{ineq:ppbs_cric}, 
\begin{align*}
    \left(\vct x^*-\vct x_k^{(\textup{B})}\right) \cdot  \left(\vct x_{k+1}^{(\textup{B})}-\vct x_k^{(\textup{B})}\right) &\geq 
    \left(\vct x^*-\vct x_k^{(\textup{B})}\right) \cdot  \vct{r}_k- \frac{M}{\rho T^{0.2}}\left|\left|\vct x^*-\vct x_k^{(\textup{B})}\right|\right|_2\\
    &\geq 0.6||\vct r_k||_2^2- \frac{M}{\rho T^{0.2}}\left|\left|\vct x^*-\vct x_k^{(\textup{B})}\right|\right|_2.
\end{align*}
Note that our assumed lower bound on $N$ implies a lower bound on $T$. Numerically, one can prove that $T^{0.2}\geq 20\rho R/M$. Hence, the above inequality implies that 
\begin{align*}
    \left(\vct x^*-\vct x_k^{(\textup{B})}\right) \cdot  \left(\vct x_{k+1}^{(\textup{B})}-\vct x_k^{(\textup{B})}\right) &\geq 0.6||\vct r_k||_2^2- \frac{0.05 M^2}{\rho^2 R}\left|\left|\vct x^*-\vct x_k^{(\textup{B})}\right|\right|_2.
\end{align*}
Then, by following the proof steps in Proposition \ref{prop:app_zeroerr}, we have that 
\begin{align}
\left|\left|\vct x^*-\vct x_{k+1}^{(\textup{B})}\right|\right|_2^2-\left|\left|\vct x^*-\vct x_k^{(\textup{B})}\right|\right|_2^2&=-2
    \left(\vct x^*-\vct x_k^{(\textup{B})}\right) \cdot  \left(\vct x_{k+1}^{(\textup{B})}-\vct x_k^{(\textup{B})}\right)+\left|\left|\vct x_{k+1}^{(\textup{B})}-\vct x_k^{(\textup{B})}\right|\right|_2^2 \nonumber\\&\leq 
    -1.2||\vct r_k||_2^2+ \frac{0.1M^2}{\rho^2 R}\left|\left|\vct x^*-\vct x_k^{(\textup{B})}\right|\right|_2+\left(\frac{M}{\rho}\right)^2,
\end{align}
where the second step is due to the construction of $\vct x_{k+1}^{(\textup{B})}$ in Algorithm \ref{alg:1gd3_opt}.
Recall that $\left|\left|\vct x^*-\vct x_0^{(\textup{B})}\right|\right|_2\leq R$. The above inequality implies that if $||\vct r_k||_2\geq{M}/{\rho}$ for all $k\leq 10R^2\rho^2/M^2$, then $\left|\left|\vct x^*-\vct x_k^{(\textup{B})}\right|\right|_2$ is non-increasing and reaches below $0$ at $k= \lfloor 10R^2\rho^2/M^2\rfloor +1$. However, this contradicts the fact that $||\vct r_k||_2$ is non-negative, and we must conclude the existence of $k_0\leq 10R^2\rho^2/M^2$ such that $||\vct r_{k_0}||_2<{M}/{\rho}$. 

~

Now consider any index $k$ with $||\vct r_{k}||_2<{M}/{\rho}$. 
By the construction of $\vct r_{k}$, we have that $\nabla f(\vct x_{k})=-\vct r_{k}\cdot \nabla^2 f(\vct x_{k})$. Then, by the Lipschitz Hessian condition, 
\begin{align}
&\left|\left|\nabla f(\vct x_{k+1}) - (\vct x_{k+1}-\vct x_{k}-\vct r_{k})\cdot \nabla^2 f(\vct x_{k})\right|\right|_2\nonumber\\=&||\nabla f(\vct x_{k+1}) - \nabla f(\vct x_{k})- (\vct x_{k+1}-\vct x_{k})\cdot \nabla^2 f(\vct x_{k}) ||_2 \nonumber\\\leq& \frac{\rho}{2} ||\vct x_{k+1}-\vct x_{k}||_2^2. 
\end{align}
Using the strong convexity assumption and triangle inequality, the above bound implies that 
\begin{align*}
&\left|\left|(\nabla^2 f(\vct x_{k+1}))^{-1}\nabla f(\vct x_{k+1})\right|\right|_2\\\leq&\, \left|\left|(\vct x_{k+1}-\vct x_{k}-\vct r_{k})\cdot \nabla^2 f(\vct x_{k})\cdot (\nabla^2 f(\vct x_{k+1}))^{-1}\right|\right|_2 +\frac{\rho}{2M} ||\vct x_{k+1}-\vct x_{k}||_2^2. 
\end{align*}
Note that the first term in the bound above is upper bounded by the product of $\left|\left|\vct x_{k+1}-\vct x_{k}-\vct r_{k}\right|\right|_2$ and the spectral norm of $\nabla^2 f(\vct x_{k})\cdot (\nabla^2 f(\vct x_{k+1}))^{-1}$. By the Lipschitz Hessian condition and strong convexity, this spectrum norm is further bounded by $1+\frac{\rho}{M}\left|\left|\vct x_{k+1}-\vct x_{k}\right|\right|_2$. Therefore, 
\begin{align}\label{pp13_final_recursion}
&\left|\left|(\nabla^2 f(\vct x_{k+1}))^{-1}\nabla f(\vct x_{k+1})\right|\right|_2\nonumber\\\leq&\, \left|\left|\vct x_{k+1}-\vct x_{k}-\vct r_{k}\right|\right|_2 \cdot\left(1+\frac{\rho}{M}\left|\left|\vct x_{k+1}-\vct x_{k}\right|\right|_2\right) +\frac{\rho}{2M} ||\vct x_{k+1}-\vct x_{k}||_2^2.
\end{align}

 We use 
 inequality \eqref{pp13_final_recursion} to bound $||\vct r_k||_2$ recursively. Assume $T$ is sufficiently large such that $T^{0.2}\geq 20$. As a rough estimate, we have
\begin{align}
&\left|\left|(\nabla^2 f(\vct x_{k+1}))^{-1}\nabla f(\vct x_{k+1})\right|\right|_2\nonumber \leq \frac{M}{20\rho} \cdot 2+\frac{M}{2\rho}\leq  0.6\frac{M}{\rho}.   
\end{align}
Recall we can find $k_0\leq 10R^2\rho^2/M^2$ such that $||\vct r_{k_0}||_2<{M}/{\rho}$. By induction, we have $||\vct r_{k}||_2\leq 0.6M/\rho$ for all $k>k_0$. 
Hence, when $k>k_0$, inequality \eqref{pp13_final_recursion} implies the following relation, where the RHS is obtained by triangle inequality and the definition of $E_{N-1}=0$. 
\begin{align}
\left|\left|\vct r_{k+1}\right|\right|_2\nonumber\leq&\, \frac{M}{\rho T^{0.2}} \cdot\left(1+\frac{\rho}{M}\left|\left|\vct r_{k}\right|\right|_2+\frac{1}{ T^{0.2}}\right) +\frac{\rho}{2M} \left(||\vct r_{k}||_2+\frac{M}{\rho T^{0.2}}\right)^2\nonumber.   
\end{align}
Therefore, by induction, we have 
\begin{align}
\left|\left|\vct r_{k}\right|\right|_2\nonumber\leq&\, \frac{M}{\rho}\max\left\{\frac{0.6}{ 2^{2^{k-k_0-2}}},\frac{2}{T^{0.2}}\right\}\nonumber   
\end{align}
for any $k> k_0+1$, and numerically, $\left|\left|\vct r_{N-1}\right|\right|_2\leq 2MT^{-0.2}/\rho$ if $T^{0.1}\geq 2k_0+6$.

\subsection{Proof of Proposition \ref{prop:13_br}}\label{app:pp13br}
\begin{proof}[Proof of inequality \eqref{ineq:pp13_main}]
We prove the inequality by considering two possible cases. In the first case, we assume that the $\ell_2$ norms of both $H^{-1}\vct m$ and  $H'^{-1}\vct m'$ are no greater than $R_0$. In this case, we have $H_{m^*}=H$ and $H_{m'^*}'=H'$. Hence, 
    \begin{align}\label{eq:ppc_eq1}
   H_{m^*}^{-1}\vct m-H_{m'^*}'^{-1}\vct m'&= H^{-1}\vct m-H'^{-1}\vct m'\nonumber\\ & =H^{-1}\left((\vct m-\vct m')+(H'-H)H'^{-1}\vct m'\right).
    \end{align}
    By the fact that all eigenvalues of $H$ are lower bounded by $M$ and the triangle inequality, 
        \begin{align}\label{eq:ppc_eq2}
    \left|\left|H_{m^*}^{-1}\vct m-H_{m'^*}'^{-1}\vct m'\right|\right|_2&\leq M^{-1}\left(||\vct m-\vct m'||_2+||H'-H||_\textup{F}||H'^{-1}\vct m'||_2\right)\nonumber\\&\leq M^{-1}\left(||\vct m-\vct m'||_2+||H'-H||_\textup{F}\cdot R_0\right).
    \end{align}
Then, the needed inequality is obtained by $ \left|\left|H_{m^*}^{-1}\vct m-H_{m'^*}'^{-1}\vct m'\right|\right|_2\leq 2R_0$, which follows from the construction of  $H_{m^*}$, $H_{m'^*}'$ and triangle inequality.

For the other case, we have $\max\left\{||H^{-1}\vct m||_2, 
||H'^{-1}\vct m'||_2\right\}> R_0$. Without loss of generality, we assume that 
$m^*\geq m'^*$. To be rigorous, here we adopted the convention that $m^*=-\infty$ if the $\ell_2$ norms of $H^{-1}\vct m$ is no greater than $R_0$, and the same for $m'^*$ accordingly. Based on this assumption, the condition in this case can be simplified as $||H^{-1}\vct m||_2>R_0$, and we have that $||H^{-1}_{m^*}\vct m||_2=R_0$. Furthermore, we also have $m^*>M$.

To prove the needed inequality, we introduce an intermediate variable $H'_{m^*}$, which is defined as the symmetric matrix sharing the eigenbasis of $H'$, but with each eigenvalue $\lambda$ replaced with $\max\{\lambda,m^*\}$. 
 Note that $H_{m^*}$ and $H'_{m^*}$ are obtained by projecting $H$ and $H'$ to a convex set of matrices under the Frobenius norm. We have that 
    \begin{align}\label{ineq_cp}
    ||H'_{m^*}-H_{m^*}||_\textup{F} \leq ||H'-H||_\textup{F}.
    \end{align}
Therefore, by following the same steps in the first case and noting that all eigenvalues of $H'_{m^*}$ are lower bounded by $m^*$, we have that 
      \begin{align*}
    \left|\left|H_{m^*}^{-1}\vct m-H_{m^*}'^{-1}\vct m'\right|\right|_2&\leq  {m^*}^{-1}\left(||\vct m-\vct m'||_2+||H'-H||_\textup{F}\cdot R_0\right).
    \end{align*}
    Compare the above to inequality \eqref{ineq:pp13_main}, it remains to prove that 
     \begin{align}\label{ineq:pp13_main_2}
    \left|\left|H_{m^*}^{-1}\vct m-H_{m'^*}'^{-1}\vct m'\right|\right|_2^2\leq \frac{2R_0m^*}{M}\cdot \left|\left|H_{m^*}^{-1}\vct m-H_{m^*}'^{-1}\vct m'\right|\right|_2. 
    \end{align}

 For brevity, we denote that 
 \begin{align*}
 \vct a&\triangleq H_{m^*}^{-1}\vct m,\\
 \vct b&\triangleq H_{m^*}'^{-1}\vct m',\\ 
 \vct c&\triangleq H_{m'^*}'^{-1}\vct m',\\ \alpha&\triangleq M/m^*.
 \end{align*}
 In the eigenbasis of $H'$, it is clear that
 \begin{align*}\left|\left|\vct b-\alpha \vct c\right|\right|_2\leq (1-\alpha) \left|\left|\vct c\right|\right|_2.
 \end{align*}
 Hence, by Cauchy's inequality,
  \begin{align}\label{ineq:ppuc_1}
  \vct a\cdot  \left(\vct b-\alpha \vct c\right )\leq \left|\left|\vct a \right|\right|_2\cdot\left|\left|\vct b-\alpha \vct c\right|\right|_2\leq (1-\alpha) \left|\left|\vct a \right|\right|_2\cdot\left|\left|\vct c\right|\right|_2.
 \end{align}
 Recall that $\left|\left|\vct c\right|\right|_2\leq R_0$ and in this case we have $\left|\left|\vct a\right|\right|_2= R_0$. Therefore, the RHS of the above inequality is upper bounded by $(1-\alpha)\left|\left|\vct a \right|\right|_2^2$, and we have
   \begin{align*}
   \vct a\cdot  \left(\vct a- \vct c\right )\leq \frac{1}{\alpha}\vct a \cdot\left(\vct a-\vct b\right)\leq \frac{1}{\alpha} R_0 \left|\left|\vct a-\vct b\right|\right|_2,
 \end{align*}
 where the first step above is equivalent to inequality \eqref{ineq:ppuc_1}, and the second step is due to Cauchy's inequality. 
Finally, it remains to notice that the LHS of inequality \eqref{ineq:pp13_main_2} equals $\left|\left|\vct a-\vct c\right|\right|_2^2$, which is upper bounded by the LHS of the above inequality, and its RHS equals the RHS of the above inequality. Hence, inequality \eqref{ineq:pp13_main_2} is proved.
\end{proof}

\begin{proof}[Proof of inequality \eqref{ineq:pp13_2}]
Firstly, if $m^*=m'^*$, we follow similar arguments from equation \eqref{eq:ppc_eq1} to inequality \eqref{eq:ppc_eq2}. I.e.,  in this case, we have
\begin{align}
  H'_{m'^*} \left(H_{m^*}^{-1}\vct m-H_{m'^*}'^{-1}\vct m'\right)
  &=(\vct m-\vct m')+(H_{m'^*}'-H_{m^*})H_{m^*}^{-1}\vct m.
    \end{align}
    Hence, by  triangle inequality and inequality \eqref{ineq_cp},  
        \begin{align}
    \left|\left|H'_{m'^*} \left(H_{m^*}^{-1}\vct m-H_{m'^*}'^{-1}\vct m'\right)\right|\right|_2
    &\leq ||\vct m-\vct m'||_2+\left|\left|H_{m'^*}'-H_{m^*}\right|\right|_\textup{F}\left|\left|H_{m^*}^{-1}\vct m\right|\right|_2\nonumber\\&\leq ||\vct m-\vct m'||_2+||H'-H||_\textup{F}\cdot R_0.
    \end{align}

Then, for $m^*>m'^*$, we define 
$H'_{m^*}$ and $\vct a$, $\vct b$, $\vct c$ as in the earlier proof steps. 
We first prove the following key inequality.
\begin{align}\label{ineq:ppc_interm65key}
||\vct a-\vct c||_2\cdot ||\vct b-\vct c||_2\leq 2 ||\vct a-\vct b||_2 \cdot ||\vct a||_2. 
\end{align}
Recall the assumption in this case implies that $||\vct a||_2=R_0$. By taking the squares on both sides, the inequality above is equivalent to the following linear inequality of vector $\vct a$. 
\begin{align}\label{ineq:ppc_interm65}
\vct a\cdot \left( 8R_0^2\, \vct b- 2  ||\vct b-\vct c||_2^2\,  \vct c \right) \leq 4R_0^2\cdot   \left(R_0^2+||\vct b||_2^2\right)- ||\vct b-\vct c||_2^2\cdot \left(R_0^2+ ||\vct c||_2^2\right). 
\end{align}
By Cauchy's inequality, the LHS of inequality \eqref{ineq:ppc_interm65} is upper bounded by 
$||\vct a||_2\cdot || 8R_0^2\, \vct b- 2  ||\vct b-\vct c||_2^2\,  \vct c ||_2$.
The coefficient of $||\vct a||_2$ in this expression can be further characterized as follows.   
\begin{align}\label{ineq:ppc_interm66}
\left|\left| 8R_0^2\, \vct b- 2  ||\vct b-\vct c||_2^2\,  \vct c \right|\right|_2^2 
=&  4\cdot \left(4R_0^2-||\vct b-\vct c||_2^2\right)\left(4R_0^2||\vct b||_2^2-||\vct b-\vct c||_2^2||\vct c||_2^2\right) \nonumber\\
&+ 16R_0^2\cdot ||\vct b-\vct c||_2^4\nonumber\\
=& \frac{1}{R_0^2} \left(4R_0^2\cdot   \left(R_0^2+||\vct b||_2^2\right)- ||\vct b-\vct c||_2^2\cdot \left(R_0^2+ ||\vct c||_2^2\right)\right)^2\nonumber\\
&-\frac{1}{R_0^2} \left(4R_0^2\cdot   \left(R_0^2-||\vct b||_2^2\right)- ||\vct b-\vct c||_2^2\cdot \left(R_0^2- ||\vct c||_2^2\right)\right)^2\nonumber\\
&+ 16R_0^2\cdot ||\vct b-\vct c||_2^4.
\end{align}
We prove that the contribution from the second term and the third term in the above expression is non-positive. 
To that end, note that the definition of $H'_{m^*}$, $H'_{m'^*}$ and the assumption of $m^*>m'^*$ imply that $(\vct c-\vct b)\cdot \vct b\geq 0$. We have the following inequalities.
\begin{align}\label{ineq:ppc_interm67}
||\vct b-\vct c||_2^2+||\vct b||_2^2\leq ||\vct c||_2^2\leq R_0^2.
\end{align}
Therefore, $0\leq 4R_0^2\cdot   \left(R_0^2-||\vct b||_2^2\right)- ||\vct b-\vct c||_2^2\cdot \left(R_0^2- ||\vct c||_2^2\right)\leq 4R_0^2\cdot   ||\vct b-\vct c||_2^2$, and  equation \eqref{ineq:ppc_interm66} implies that 
\begin{align}
\vct a\cdot \left( 8R_0^2\, \vct b- 2  ||\vct b-\vct c||_2^2\,  \vct c \right)\leq  \left|4R_0^2\cdot   \left(R_0^2+||\vct b||_2^2\right)- ||\vct b-\vct c||_2^2\cdot \left(R_0^2+ ||\vct c||_2^2\right)\right|.\nonumber
\end{align}
By utilizing the above bound, inequality \eqref{ineq:ppc_interm65} is proved by noting that its RHS is non-negative, which can be proved using inequality \eqref{ineq:ppc_interm67}. As mentioned earlier, this implies inequality \eqref{ineq:ppc_interm65key}. 

To proceed further, we note that 
$\vct b-\vct c$ lies in the eigenspace of $H_{m^*}'$ associated with eigenvalue $m^*$. Hence, 
\begin{align}
   H_{m^*}' 
    \left(\vct a-\vct c\right)=H_{m^*}' 
    \left(\vct a-\vct b\right)+m^* \left(\vct b-\vct c\right).\nonumber
\end{align}
Therefore, by triangle inequality, we have
\begin{align}\label{ineq:ppc_interm66l}
\left|\left|H'_{m^*} 
    \left(H_{m^*}^{-1}\vct m-H_{m'^*}'^{-1}\vct m'\right)\right|\right|_2
    \leq& \left|\left|H_{m^*}' 
    \left(\vct a-\vct b\right)\right|\right|_2+m^* \left|\left|\vct b-\vct c\right|\right|_2  .
\end{align}
Note that by inequality \eqref{ineq:ppc_interm65key} and the fact that all eigenvalues of $H_{m^*}'$ are lower bounded by $m^*$, we have 
\begin{align}
m^* \left|\left|\vct b-\vct c\right|\right|_2\leq \frac{2R_0}{||\vct a-\vct c||_2}  \left|\left|
   H_{m^*}'
    \left(\vct a-\vct b\right)\right|\right|_2.\nonumber
\end{align}
Therefore, it remains to upper bound the $\ell_2$ norm of $ H_{m^*}'
    \left(\vct a-\vct b\right)$.

By the definition of vectors $\vct a$, $\vct b$, 
\begin{align}
 H_{m^*}'
    \left(\vct a-\vct b\right) 
  &=(\vct m-\vct m')+(H_{m^*}'-H_{m^*})\, \vct a.
    \end{align}
The triangle inequality implies that
\begin{align}
   \left|\left|
   H_{m^*}'
    \left(\vct a-\vct b\right)\right|\right|_2&\leq  \left|\left|\vct m-\vct m'\right|\right|_2+  R_0 \left|\left|H_{m^*}'-H_{m^*}\right|\right|_{\textup{F}}. 
\end{align}
Hence,  
\begin{align*}
\left|\left|H'_{m^*} 
    \left(H_{m^*}^{-1}\vct m-H_{m'^*}'^{-1}\vct m'\right)\right|\right|_2
    \leq& \left(1+\frac{2R_0}{||\vct a-\vct c||_2}\right)\\
    &\cdot \left(  \left|\left|\vct m-\vct m'\right|\right|_2+R_0 \left|\left|H_{m^*}'-H_{m^*}\right|\right|_{\textup{F}}\right)
    \\\leq& \left(1+\frac{2R_0}{||H_{m^*}^{-1}\vct m-H_{m'^*}'^{-1}\vct m'||_2}\right)\nonumber\\
    &\cdot \left(  \left|\left|\vct m-\vct m'\right|\right|_2+R_0 \left|\left|H-H'\right|\right|_{\textup{F}}\right), 
\end{align*}
where the last step is due to inequality \eqref{ineq_cp}. Thus, inequality \eqref{ineq:pp13_2} is implied by the semi-positive-definiteness of $H'_{m^*}-H'_{m'^*}$.


Finally, when $m^*<m'^*$,  
we let $H_{m'^*}$ denote the symmetric matrix sharing the same eigenbasis of $H$, but with each eigenvalue $\lambda$ replaced by $\max\{\lambda, m'^*\}$. Due to the equivalence of $H$ and $H'$, our earlier proof steps imply that 
\begin{align}
\left|\left|H_{m*} 
    \left(H_{m^*}^{-1}\vct m-H_{m'^*}'^{-1}\vct m'\right)\right|\right|_2
    \leq& \left(1+\frac{2R_0}{||H_{m^*}^{-1}\vct m-H_{m'^*}'^{-1}\vct m'||_2}\right)\nonumber\\
    &\cdot \left(  \left|\left|\vct m-\vct m'\right|\right|_2+R_0 \left|\left|H-H'\right|\right|_{\textup{F}}\right).\nonumber
\end{align}
Hence, by triangle inequality, we can use the above bound as follows. 
\begin{align}
\left|\left|H'_{m'^*} 
    \left(H_{m^*}^{-1}\vct m-H_{m'^*}'^{-1}\vct m'\right)\right|\right|_2\leq &\left|\left|\left(H_{m'^*}-H'_{m'^*}\right) 
    \left(H_{m^*}^{-1}\vct m-H_{m'^*}'^{-1}\vct m'\right)\right|\right|_2\nonumber\\&+\left|\left|H_{m'^*} 
    \left(H_{m^*}^{-1}\vct m-H_{m'^*}'^{-1}\vct m'\right)\right|\right|_2\nonumber\\
    \leq &\left|\left|H_{m'^*}-H'_{m'^*}\right|\right|_{\textup{F}} 
    \left|\left|H_{m^*}^{-1}\vct m-H_{m'^*}'^{-1}\vct m'\right|\right|_2\nonumber\\&+\left|\left|H_{m'^*} 
    \left(H_{m^*}^{-1}\vct m-H_{m'^*}'^{-1}\vct m'\right)\right|\right|_2\nonumber.
\end{align}
Note that $\left|\left|H_{m^*}^{-1}\vct m-H_{m'^*}'^{-1}\vct m'\right|\right|_2\leq 2R_0$.  
By inequality \eqref{ineq_cp}, it is clear that 
\begin{align*}
\left|\left|H'_{m'^*} 
    \left(H_{m^*}^{-1}\vct m-H_{m'^*}'^{-1}\vct m'\right)\right|\right|_2
    \leq&  \left(3+\frac{2R_0}{||H_{m^*}^{-1}\vct m-H_{m'^*}'^{-1}\vct m'||_2}\right)\nonumber\\
    &\cdot \left(  \left|\left|\vct m-\vct m'\right|\right|_2+R_0 \left|\left|H-H'\right|\right|_{\textup{F}}\right).
\end{align*}
    \end{proof}

\end{document}